\documentclass[journal]{IEEEtran}

\ifCLASSINFOpdf
\else
\fi

\usepackage{cite}
\usepackage{algorithmic}
\usepackage{algorithm}
\usepackage{cite}
\usepackage{hyperref}
\usepackage{multicol}
\usepackage{mathtools}
\usepackage{amssymb,latexsym,amsmath,mathtools,bm,amsthm}
\usepackage{graphicx,color}
\usepackage[dvipsnames]{xcolor}
\usepackage{bbm,eufrak,mathrsfs}
\usepackage[keeplastbox]{flushend}
\usepackage{csquotes}
\usepackage{subcaption}
\usepackage{enumitem}
\usepackage{multirow}
\usepackage[normalem]{ulem}
\usepackage{setspace}

\newcommand{\code}[1]{\texttt{#1}}

\newtheorem{lem}{Lemma} 
\newtheorem{thm}{Theorem}
\newtheorem{cor}{Corollary}
\newtheorem{assum}{Assumption}
\newtheorem{Def}{Definition}
\newtheorem{prop}{Proposition}

\def\mb{\mathbb}
\def\mbf{\mathbf}
\def\mc{\mathcal}
\def\mk{\mathfrak}

\def\m{\mathbf}

\DeclarePairedDelimiter{\ceil}{\lceil}{\rceil}
\DeclarePairedDelimiter\floor{\lfloor}{\rfloor}
\newcommand{\ip}[2]{\left\langle #1, #2 \right\rangle}
\newcommand{\br}[1]{\left[#1 \right]}
\newcommand{\pr}[1]{\left(#1 \right)}
\newcommand{\nor}[1]{\left\|#1 \right\|}
\newcommand{\cl}{\mathop{\mathrm{cl}}\nolimits}
\newcommand{\tr}{\mathop{\mathrm{Tr}}\nolimits}

\newcommand{\vect}{\mathop{\mathrm{vec}}\nolimits}

\newcommand{\sgn}{\mathop{\mathrm{sgn}}\nolimits}
\newcommand{\argmin}{\mathop{\mathrm{argmin}}}

\newcommand{\E}{\mb{E}}

\newcommand{\mxc}{2,\infty}
\newcommand{\sumc}{1,1}

\newcommand{\stark}{{STARK}}
\newcommand{\subDL}{SubDil}
\newcommand{\CPbased}{TeFDiL}
\newcommand{\const}{\mathop{\mathrm{const.}}\nolimits}
\newcommand{\Dp}{\t D^{\pi}}
\newcommand{\shrink}{\mathop{\mathrm{shrink}}\nolimits}

\newcommand{\refold}{\mathop{\mathrm{refold}}\nolimits}
\newcommand{\rank}{\mathop{\mathrm{rank}}\nolimits}

\newcommand{\strnorm}[1]{\nor{#1}_{\mathrm{str}}}
\newcommand{\trnorm}[1]{\nor{#1}_{\mathrm{tr}}}

\newcommand{\sr}{\mk R^{N}_{\v m,\v p}}
\newcommand{\knr}{\mathcal K^{N,r}_{\v m,\v p}}
\newcommand{\cknr}{{}^c{\mathcal K}^{N,r}_{\v m,\v p}}

\newcommand{\knrClosed}{\overline{\mathcal K}^{N,r}_{\v m,\v p}}
\newcommand{\kn}{\mathcal K^{N}_{\v m,\v p}}
\newcommand{\kr}{\mathcal K^{2,r}_{\v m,\v p}}
\newcommand{\knrall}{\mathcal L^{N,r}_{\v m,\v p}}
\newcommand{\cknrall}{{}^c{\mathcal L}^{N,r}_{\v m,\v p}}

\newcommand{\knrallClosed}{\overline{\mathcal L}^{N,r}_{\v m,\v p}}
\newcommand{\lsr}{\mathcal C}
\newcommand{\bro}{\mathcal B_{\rho}}
\newcommand{\ball}{\mathcal U}
\newcommand{\fp}{f_{\mc P}}
\newcommand{\fy}{F_{\m Y}}
\newcommand{\fyfac}{F^{\mathrm{fac}}_{\Y}}
\newcommand{\fyreg}{F^{\mathrm{reg}}_{\Y}}

\newcommand{\fpreg}{f^{\mathrm{reg}}_{\mc P}}
\newcommand{\ellreg}{\ell^{\mathrm{reg}}_{\Y}}
\newcommand{\N}{\mc N}

\newcommand{\x}{\v x}
\newcommand{\X}{\v X}
\newcommand{\y}{\v y}
\newcommand{\Y}{\v Y}
\newcommand{\M}{\m M}

\newcommand{\D}{\m D}
\newcommand{\s}{\v s}
\newcommand{\dkn}{\D^k_n}
\newcommand{\dknset}{\{\D^k_n\}}
\newcommand{\dkntup}{(\D^k_n)}

\renewcommand{\v}[1]{\mbf{#1}}
\renewcommand{\t}[1]{\underline{\mbf{#1}}}
\renewcommand{\m}[1]{\v{#1}}

\newcommand\redout{\bgroup\markoverwith{\textcolor{red}{\rule[.5ex]{2pt}{0.4pt}}}\ULon}

\begin{document}
%
\title{Learning Mixtures of Separable Dictionaries for Tensor Data: Analysis and Algorithms}
%
%
%

\author{Mohsen~Ghassemi,~\IEEEmembership{Student~Member,~IEEE,}~Zahra~Shakeri,~\IEEEmembership{Member,~IEEE,}\\Anand~D.~Sarwate,~\IEEEmembership{Senior~Member,~IEEE,}~and~Waheed~U.~Bajwa,~\IEEEmembership{Senior~Member,~IEEE}
\thanks{M. Ghassemi, A.D. Sarwate, and W.U. Bajwa are with the Department of Electrical and Computer Engineering, Rutgers, The State University of New Jersey, Piscataway, NJ, 08854 USA (Emails: \{m.ghassemi, anand.sarwate, waheed.bajwa\}@rutgers.edu). Z. Shakeri was with the Department of Electrical and Computer Engineering, Rutgers, The State University of New Jersey, Piscataway, NJ, 08854 USA and is now with Electronic Arts (Email: zshakeri@ea.com)}
\thanks{Some of the results reported here were presented at 
CAMSAP 2017~\cite{ghassemi2017stark} and
ISIT 2019 \cite{ghassemi19isit}. This work is supported in part by the U.S.~National Science Foundation under awards CCF-1453073 and CCF-1910110 and by the U.S.~Army Research Office under award W911NF-17-1-0546.}%
}%

%
%


\maketitle

\begin{abstract}
This work addresses the problem of learning sparse representations of tensor data using structured dictionary learning. It proposes learning a mixture of separable dictionaries to better capture the structure of tensor data by generalizing the separable dictionary learning model. Two different approaches for learning mixture of separable dictionaries are explored and sufficient conditions for local identifiability of the underlying dictionary are derived in each case. Moreover, computational algorithms are developed to solve the problem of learning mixture of separable dictionaries in both batch and online settings. Numerical experiments are used to show the usefulness of the proposed model and the efficacy of the developed algorithms.
\end{abstract}


\IEEEpeerreviewmaketitle

\section{Introduction}
Many data processing tasks such as feature extraction, data compression, classification, signal denoising, image inpainting, and audio source separation use sparse representations learned from data~\cite{kreutz2003dictionary,Tosic2011dictionary,aharon2006img}. In many cases, these applications also involve data samples that are naturally structured as multiway arrays, also known as multidimensional arrays or tensors. Instances of \textit{multidimensional} or \textit{tensor} data include videos, hyperspectral images, tomographic images, and multiple-antenna wireless channels. Despite the ubiquity of tensor data in many applications, traditional data-driven sparse representation approaches disregard their multidimensional structure. This can result in sparsifying models with a large number of parameters. With the increasing availability of large data sets, it is crucial to keep sparsifying models reasonably small to ensure their scalable learning and efficient storage within devices such as smartphones and drones.

Our focus in this paper is on learning of ``compact'' models that yield sparse representations of tensor data. To this end, we study \emph{dictionary learning} (DL) for tensor data. The goal in DL, which is an effective and popular data-driven technique for obtaining sparse representations of data~\cite{kreutz2003dictionary,aharon2006img,Tosic2011dictionary}, is to learn a dictionary $\D$ such that every data sample can be approximated by a linear combination of a few atoms (columns) of $\D$. While DL has been widely studied, traditional DL approaches flatten tensor data and then employ methods designed for vector data~\cite{aharon2006img,mairal2010online}. This ignores the multidimensional structure in tensor data, resulting in dictionaries with a large number of parameters. One intuitively expects that dictionaries which exploit the correlation and structure across tensor modes will have fewer parameters, leading to improvements in storage requirements, computational complexity, and  generalization performance, especially when training data are noisy or scarce.

\begin{figure}
\centering
 \includegraphics[width=0.49\textwidth]{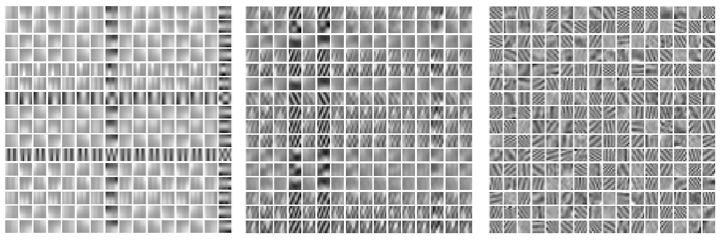}
\caption{Dictionary atoms for representing RGB image \texttt{Barbara} for separation rank (left-to-right) $1$, $4$, and $256$.}
\label{figure_atoms}
\vspace{-\baselineskip}
\end{figure}

To reduce the number of parameters in dictionaries for tensor data, and to better exploit the correlation among different tensor modes, some recent DL works use tensor decompositions such as the Tucker decomposition~\cite{tucker1963implications} and CANDECOMP/PARAFAC decomposition (CPD)~\cite{harshman1970foundations} for learning of ``structured'' dictionaries. The idea in \emph{structured DL} for tensor data is to restrict the class of dictionaries during training to the one imposed by the tensor decomposition under consideration~\cite{Shakeri2019Book}. For example, structured DL based on the Tucker decomposition of $N$-way tensor data corresponds to the dictionary class in which any dictionary $\D\in \mb R^{m \times p}$ consists of the Kronecker product~\cite{van2000ubiquitous} of $N$ smaller \textit{subdictionaries} $\{\D_n\in \mb R^{m_n\times p_n}\}_{n=1}^N$~\cite{hawe2013separable,roemer2014tensor,dantas2017learning,zubair2013tensor,shakeri2016minimax,shakeri2018achieve}. The resulting DL techniques in this instance are interchangeably referred to in the literature as \textit{separable DL} or \textit{Kronecker-structured DL} (KS-DL).
%

In terms of parameters to be estimated and stored, the advantages of KS-DL for tensor data are straightforward:  defining $m\triangleq\prod_{n=1}^N m_n$ and $p\triangleq\prod_{n=1}^N p_n$, unstructured dictionary learning uses $mp=\Pi_{n=1}^N m_n  p_n$ parameters, whereas the KS-DL model uses only the sum of the subdictionary sizes $\sum_{n=1}^N m_np_n$. %
Nonetheless, while existing KS-DL methods enjoy lower sample/computational complexity and better storage efficiency over unstructured DL~\cite{shakeri2018achieve}, the KS-DL model makes a strong separability assumption among different modes of tensor data. Such an assumption can be overly restrictive for many classes of data~\cite{Noel99nonseparable}, resulting in an unfavorable tradeoff between model compactness and representation power.

In this paper, we overcome this limitation by proposing and analyzing a generalization of KS-DL that we interchangeably refer to as \textit{learning a mixture of separable dictionaries} or \textit{low separation rank DL} (LSR-DL). The separation rank of a matrix $\m A$ is defined as the minimum number of KS matrices whose sum equals $\m A$ \cite{Beylkin2002lsr,hero_2013_kronecker}. The LSR-DL model interpolates between the under-parameterized separable model (a special case of LSR-DL model with separation rank $1$) and the over-parameterized unstructured model. In particular, this model is a natural and consistent way to increase the number of parameters in structured DL (and therefore representation performance) while mimicking the compactness of the KS-DL model. Our numerical experiments confirm the advantages of the LSR-DL model (and algorithms) over both unstructured DL and KS-DL in terms of sample complexity/performance and the amount of memory needed to store the dictionary.  Figure~\ref{figure_atoms} also illustrates the difference between LSR-DL and KS-DL: while KS-DL learns dictionary atoms that cannot reconstruct diagonal structures perfectly because of the abundance of  {axis-aligned (horizontal/vertical)} structures within them, LSR-DL also returns dictionary atoms with pronounced diagonal structures as the separation rank increases.

\subsection{Main Contributions}

We propose a new generalization of the  separable DL model---which we call a mixture of separable dictionaries model or LSR-DL model---that has a smaller number of parameters than standard DL. We provide conditions under which a true dictionary is recoverable, up to a prescribed error, from tensor-valued training data generated from the LSR-DL model. Our analysis uses the conventional optimization-based formulation of the DL problem~\cite{Tosic2011dictionary}, except that the search space is constrained to the class of dictionaries with maximum separation rank $r$ (and individual mixture terms having bounded norms when $N \geq 3$ and $r \geq 2$).\footnote{We also provide asymptotic identifiability results for LSR dictionaries without requiring the boundedness assumption; 
see Section~\ref{sec:NP_identif} for details.} Due to our choice of Frobenius norm as the distance metric, similar to conventional DL problems, this LSR-DL problem is nonconvex with multiple global minima. 
Obtaining global convergence guarantees for this highly nonconvex problem is not only difficult but is also insufficient to guarantee global identifiability due to existence of multiple global minima. We therefore focus on \emph{local identifiability} guarantees, meaning that a search algorithm initialized close enough to the true dictionary can recover that dictionary. Our local identifiability results show that the LSR-DL problem is well posed---i.e., it can return a good estimate of the true dictionary, up to a certain initialization distance, as a solution---and characterize the effect of the separation rank on the sample complexity of the learning problem. To this end, under certain assumptions on the generative model, we show that $\Omega\big(r(\sum_{n=1}^N m_n p_n) p^2 \rho^{-2}\big)$ samples ensure existence of a local minimum of the constrained LSR-DL problem for $N$th-order tensor data 
within a neighborhood of radius $\rho$ around the true LSR dictionary.

Our initial local identifiability results are based on an analysis of a separation rank-constrained optimization problem that exploits a connection between LSR (resp., KS) matrices and low-rank (resp., rank-1) tensors. The main challenge that we face as part of this analysis is understanding the topological properties of the class of dictionaries with separation rank at most $r$ in terms of compactness and covering number. The resulting insights may be of independent interest to readers for other problems involving LSR matrices.

Next, we note that a result in tensor recovery literature~\cite{haastad1990tensor} implies finding the separation rank of a matrix is NP-hard. While this means that the rank-constrained LSR-DL problem is computationally intractable, our analysis of this problem provides the basis for our second main contribution, which is development and analysis of two different relaxations of the LSR-DL problem that are computationally tractable in the sense that they do not require explicit computation of the separation rank. The first formulation once again exploits the connection between LSR matrices and low-rank tensors and uses a convex regularizer to implicitly constrain the separation rank of the learned dictionary. The second formulation enforces the LSR structure on the dictionary by explicitly writing it as a summation of $r$ KS matrices. Our analyses of the two relaxations once again involve conditions under which the true LSR dictionary is locally recoverable from training tensor data.
Our strongest result is for the factorized formulation, described formally in Section~\ref{sec:problem_statement}, in which case we derive a sample complexity result that is similar to that of the intractable formulation by finding a correspondence between the local minima of the factorized problem and those of the intractable one. We also compare and contrast the three sets of identifiability results for LSR dictionaries in the body.

In addition to showing the well-posedness of the LSR-DL problem, our theoretical results on local identifiability confirm the advantages of exploiting low-rank structure in tensor problems. Moreover, in order to obtain our results, we acquire a better understanding of the topological properties of the space of LSR matrices. These properties may be useful in other works involving KS and LSR models.

%
%

Our third main contribution is the development of practical computational algorithms, which are based on the two relaxations of LSR-DL, for learning of an LSR dictionary in both batch and online settings. We use these algorithms for learning of LSR dictionaries for both synthetic and real tensor data and show their effectiveness in denoising and representation learning tasks. Numerical results obtained as part of these efforts help validate the usefulness of our proposed LSR-DL model and highlight the different strengths and weaknesses of the two LSR-DL relaxations and the corresponding algorithms. In particular, we show empirically that our algorithms provide better representations with a smaller number of parameters than the conventional approach.


%

%

\subsection{Relation to Prior Work}
Tensor decompositions~\cite{kolda_tensor,oseledets2011tensor} are an important tool for avoiding overparameterization of tensor data models in a variety of areas. These include deep learning, 
collaborative filtering, 
multilinear subspace learning, 
source separation, 
topic modeling, 
and many other works (see recent surveys~\cite{novikov2015tensorizing,sidiropoulos2017tensor}~and references therein). However, the use of tensor decompositions for reducing the (model and sample) complexity of dictionaries for tensor data has been addressed only recently. 
%

Many recent works provide theoretical analysis for the sample complexity of the conventional DL problem~\cite{arora2013new,agarwal2013learning,gribonval2014sparse,schnass2014identifiability}. Among these, Gribonval et al.~\cite{gribonval2014sparse} focus on the local identifiability of the true dictionary underlying vectorized data using Frobenius norm as the distance metric. 
Shakeri et al.~\cite{shakeri2018achieve} extended this analysis for the sample complexity of the KS-DL problem for $N$th-order tensor data. 
This analysis relies on expanding the objective function in terms of subdictionaries and exploiting the coordinate-wise Lipschitz continuity property of the objective function with respect to each subdictionary~\cite{shakeri2018achieve}. While this approach ensures the identifiability of the subdictionaries, it requires the dictionary coefficient vectors to follow the so-called \emph{separable sparsity model}~\cite{caiafa2013multidimensional} and does not extend to the LSR-DL problem. 
By contrast, we provide local identifiability sample complexity results for the LSR-DL problem and its two relaxations. 
Further, our identifiability results hold for coefficient vectors following both the random and separable sparsity models.

In terms of computational algorithms, several works have proposed methods for learning KS dictionaries that rely on alternating minimization techniques to update the subdictionaries~\cite{caiafa2013multidimensional,zubair2013tensor,roemer2014tensor}. Among other works, Hawe et al.~\cite{hawe2013separable} employ a Riemannian conjugate gradient method combined with a nonmonotone line search for KS-DL. While they present the algorithm only for matrix data, its extension to higher-order tensor data is trivial. Schwab et al.~\cite{schwab2018separable}~have also recently addressed the separable DL problem for matrix data; their contributions include a computational algorithm and global recovery guarantees. In terms of algorithms for LSR-DL, Dantas et al.~\cite{dantas2017learning} proposed one of the first methods for matrix data that uses a convex regularizer to impose LSR on the dictionary. One of our batch algorithms, named \stark~\cite{ghassemi2017stark}, also uses a convex regularizer for imposing LSR structure. In contrast to Dantas et al.~\cite{dantas2017learning}, however, \stark~can be used to learn a dictionary from tensor data of any order. The other batch algorithm we propose, named \CPbased, learns subdictionaries of the LSR dictionary by exploiting the connection to tensor recovery and using tensor CPD. Recently, Dantas et al.~\cite{dantas2018learning} proposed an algorithm for learning an LSR dictionary for tensor data in which the dictionary update stage
is a projected gradient descent algorithm that involves a CPD after every gradient step. In contrast, \CPbased~only requires a single CPD at the end of each dictionary update stage. Finally, while there exist a number of online algorithms for DL~\cite{mairal2010online,skretting2010recursive,Dohmatob2016brain}, the online algorithm developed in here is the first one that enables learning of structured (either KS or LSR) dictionaries.

\subsection{Organization}
In Section \ref{sec:problem_statement}, we provide the necessary background on dictionary learning, introduce our LSR-DL model, and formulate three variants of the LSR-DL problem. In Section \ref{sec:NP_identif}, we show that LSR dictionaries are identifiable using the rank-constrained formulation of the LSR-DL problem.
%
%
In Section \ref{sec:tractable_identif}, we study the local identifiability of the other two (regularized and factorized) formulations in both asymptotic and finite sample regimes. 
In Section \ref{sec:algorithms}, we use the regularized and factorized formulations to design batch and online LSR-DL algorithms, which we evaluate experimentally in Section \ref{sec:experiments}. We conclude the paper and discuss possible future work in
Section \ref{sec:conc}. Proof of Lemma~\ref{lem:rearrangement} (the rearrangement procedure) is explained in detail in Appendix~\ref{rearrangement_section}. Proofs of technical Lemmas \ref{lem:closedness_LSR}, \ref{lem:unbounded_sequence}, \ref{lem:covering_number_2nd}, \ref{lem:covering_number_general}, and \ref{lem:neighborhoods_relation_D_to_Dkn}
are provided in Appendix~\ref{proofs_section}. 
A discussion on the convergence of our algorithms is provided in Appendix~\ref{sec:convergence}. 

\section{Preliminaries and Problem Statement} \label{sec:problem_statement}
\textbf{Notation and Definitions:} We use underlined bold upper-case ($\t A$), bold upper-case ($\m A$), bold lower-case ($\v a$), and lower-case ($a$) letters to denote tensors, matrices, vectors, and scalars, respectively. For any integer $p$, we define $[p]\triangleq\{1,2,\cdots,p\}$. We denote the $j$-th column of a matrix $\m A$ by $\v a_j$. For an $m\times p$ matrix $\m A$ and an index set $\mc J\subseteq [p]$, we denote the matrix constructed from the columns of $\m A$ indexed by $\mc J$ as $\m A_{\mc J}$. We denote by $(\m A_n)_{n=1}^N$ an $N$-tuple $(\m A_1,\cdots,\m A_N)$, while $\{\m A_n\}_{n=1}^N$ represents the set $\{\m A_1,\cdots,\m A_N\}$. We drop the range indicators if they are clear from the context.

\textit{Norms and inner products:} We denote by $\|\v v\|_p$ the $\ell_p$ norm of vector $\v v$ (we abuse the terminology in case of $p=0$), while we use $\|\m A\|_2$, $\|\m A\|_F$, and $\trnorm{\m A}$ to denote the spectral, Frobenius, and trace (nuclear) norms of matrix $\m A$, respectively. Moreover, $\|\m A\|_{\mxc}\triangleq\max_j \|\v a_j\|_2$ is the \textit{max column norm} and $\|\m A\|_{1,1}\triangleq\sum_j \|\v a_j\|_1$. We define the inner product of two tensors (or matrices) $\t A$ and $\t B$ as $\ip{\t A}{\t B}\triangleq \ip{\vect(\t A)}{\vect(\t B)}$ where $\vect(\cdot)$ is the vectorization operator. We define the Frobenius norm of tensor $\t A$ as $\|\t A\|_F=\sqrt{\ip{\t A}{\t A}}$. The Euclidean distance between two tuples of the same size is defined as $\big\|(\m A_n)_{n=1}^N-(\m B_n)_{n=1}^N\big\|_F\triangleq\sqrt{\sum_{n=1}^N ||\m A_n-\m B_n||_F^2}$.

\textit{Kronecker product:} We denote by $\m A\otimes \m B\in \mb R^{m_1m_2\times p_1p_2}$ the Kronecker product of matrices $\m A\in \mb R^{m_1\times p_1}$ and $\m B\in \mb R^{m_2\times p_2}$. We use $\bigotimes_{n=1}^N \m A_n \triangleq \m A_1\otimes\m A_2\otimes\cdots\otimes\m A_N$ for the Kronecker product of $N$ matrices. We drop the range indicators when there is no ambiguity. We call a matrix a ($N$-th order) Kronecker-structured (KS) matrix if it is a Kronecker product of $N\geq 2$ matrices.

\textit{Definitions for matrices:}
For a matrix $\D$ with unit $\ell_2$-norm columns, we define the \textit{cumulative coherence} $\mu_s (\D)$ as $\mu_s(D) \triangleq \max_{|\mc J|\leq s} \max_{j\notin \mc J} \|\D_{\mc J}^T \v d_j\|_1$. We say a matrix $\D$ satisfies the \textit{$s$-restricted isometry property} ($s$-RIP) with constant $\delta_s$ if for any $\v v \in \mb R^s$ and any $\mc J\subseteq [p]$ with $|\mc J|\leq s$, we have $(1-\delta_s)\|\v v\|_2^2\leq \|\D_{\mc J}\v v\|_2^2 \leq (1+\delta_s)\|\v v\|_2^2$.

\textit{Definitions for tensors:} We briefly present required tensor definitions here: see Kolda and Bader~\cite{kolda_tensor} for more details. The mode-$n$ unfolding matrix of $\t A$ is denoted by $\m A_{(n)}$, where each column of $\m A_{(n)}$ consists of the vector formed by fixing all indices of $\t A$ except the one in the $n$th-order. We denote the outer product (tensor product) of vectors by $\circ$, while $\times_n$ denotes the mode-$n$ product between a tensor and a matrix. An $N$-way tensor is rank-$1$ if it can be written as outer product of $N$ vectors: $\v v_1\circ\cdots\circ\v v_N$. Throughout this paper, by the rank of a tensor, $\rank(\t A)$, we mean the CP-rank of $\t A$, the minimum number of rank-$1$ tensors that construct $\t A$ as their sum. The \emph{CP decomposition} (CPD),  decomposes a tensor into sum of its rank-$1$ tensor components. The \textit{Tucker decomposition} factorizes an $N$-way tensor $\t A\in \mb{R}^{m_1 \times m_2 \times \cdots \times m_N}$ as $\t A= \t X \times_1 \m D_{1} \times_2 \m D_{2} \times_3 \cdots \times_N \m D_{N}$, where $\t X\in \mb{R}^{p_1 \times p_2 \times \cdots \times p_N}$ denotes the core tensor and $\m D_n\in \mb R^{m_n\times p_n}$ denote factor matrices along the $n$-th mode of $\t A$ for $n\in[N]$.

\textit{Notations for functions and spaces:}
We denote the element-wise sign function by $\sgn(\cdot)$.
For any function $f(\x)$, we define the difference $\Delta f(\v x_1;\x_2)\triangleq f(\x_1)-f(\x_2)$. We denote by $\ball_{m\times p}$ the Euclidean unit sphere:  $\ball_{m\times p}\triangleq\{\D\in \mb R^{m\times p}|\|\D\|_F=1\}$. We also denote the Euclidean sphere with radius $\alpha$ by $\alpha \ball_{m\times p}$. The oblique manifold in $\mb R^{m\times p}$ is the manifold of matrices with unit-norm columns: $\mc D_{m\times p}\triangleq\{\D\in \mb R^{m\times p}| \forall j\in [p],~ \v d_j^T\v d_j=1\}$. We drop the dimension subscripts and use only $\mc D$ when there is no ambiguity. The covering number of a set $\mc A$ with respect to a norm $\|\cdot\|_*$, denoted by $\N_*(\mc A,\epsilon)$, is the minimum number of balls of $*$-norm radius $\epsilon$ needed to cover $\mc A$.

\textbf{Dictionary Learning Setup:}
In dictionary learning (DL) for vector data, we assume observations $\y \in \mb R^{m}$ are generated according to the following model:
\begin{align} \label{eq:conv_DL}
\y = \D^0\x^0 + \bm \epsilon,
\end{align}
where $\D^0 \in \mc D_{m\times p} \subset \mb R^{m\times p} $ is the true underlying dictionary, $\x^0 \in \mb R^p$ is a randomly generated sparse coefficient vector, and $\bm \epsilon \in \mb R^m$ is the observation noise vector. The goal in DL is to recover the true dictionary given the noisy observations $\Y \triangleq \{\y_l\}_{l=1}^L$ that are independent realizations of \eqref{eq:conv_DL}. The ideal objective is to solve the statistical risk minimization problem
\begin{align}\label{problem:expected_risk_minimization_general}
\min_{\D\in \mc C} ~
\fp(\D)\triangleq\E_{\y \sim \mc P} ~f_{\y}(\D),
\end{align}
where $\mc P$ is the underlying distribution of the observations, $\mc C\subseteq \mc D_{m \times p}$ is the dictionary class, typically selected for vector data to be the same as the oblique manifold, and
\begin{align}\label{sparse_coding_general}
f_{\y}(\D)\triangleq\inf_{\x\in \mb R^p} \frac{1}{2} \nor{\y-\D\x}_2^2+\lambda \|\x\|_1.
\end{align}
%
However, since we have access to the distribution $\mc P$ only through noisy observations drawn from this distribution, we resort to solving the following empirical risk minimization problem as a proxy for Problem \eqref{problem:expected_risk_minimization_general}:
\begin{align}\label{problem:empirical_risk_minimization_general}
\min_{\D\in \mc C} ~\fy(\D)\triangleq\frac{1}{L}\sum\nolimits_{l=1}^L f_{\y_l}(\D).
\end{align}

\textbf{Dictionary Learning for Tensor Data:}  To represent tensor data, conventional DL approaches vectorize tensor data samples and treat them as one-dimensional arrays. 
One way to explicitly account for the tensor structure in data is to use the Kronecker-structured DL (KS-DL) model, which is based on the Tucker decomposition of tensor data.
In the KS-DL model, we assume that observations $\underline{\Y}_l \in \mb R^{m_1\times \dots \times m_N}$ are generated according to
\begin{align}\label{KS_DL_gen}
\t Y_l= \t X^0_l \times_1 \m D^0_{1} \times_2 \m D^0_{2} \times_3 \cdots \times_N \m D^0_{N}+\t {\mc E}_l,
\end{align}
where $\{\D_n^0 \in \mb R^{m_n\times p_n}\}_{n=1}^N$ are generating \textit{subdictionaries}, and $\t X^0_l$ and $\t {\mc E}_l$ are the coefficient and noise tensors, respectively. 
Equivalently, the generating model \eqref{KS_DL_gen} can be stated for $\y_l \triangleq \vect(\underline{\Y}_l)$ as:
\begin{align} \label{KS_DL_gen_vec}
 \y_l = \pr{\m D_{N}^0 \otimes \m D_{N-1}^0\otimes \cdots \otimes \m D_{1}^0} \m \x^0_l + \bm \epsilon_l,
 \end{align}
where $\x^0_l \triangleq \vect(\underline{\X^0}_l)$ and $\bm \epsilon_l \triangleq \vect(\t{\mc E}_l)$ \cite{kolda_tensor}.
%
This is the same as the unstructured model $\y_l=\D^0\x^0_l+\bm \epsilon_l$ with the additional condition that the generating dictionary is a Kronecker product of $N$ subdictionaries. 
As a result, in the KS-DL problem, the constraint set in \eqref{problem:empirical_risk_minimization_general} becomes $\mc C=\kn$, where  $\kn\triangleq\{\D\in \mc D_{m\times p}| \D=\bigotimes\nolimits_{n=1}^N \D_n,~\D_n\in \mb R^{m_n\times p_n}\}$ is the set of KS matrices with unit-norm columns and $\v m$ and $\v p$ are vectors containing $m_n$'s and $p_n$'s, respectively.\footnote{We have changed the indexing of subdictionaries for ease of notation.}

In summary, the structure in tensor data is exploited in the KS-DL model by assuming the dictionary is ``separable'' into subdictionaries for each mode. However, as discussed earlier, this separable model is rather restrictive. Instead, we generalize the KS-DL model using the notion of \textit{separation rank}.\footnote{The term was introduced in Tsiligkaridis and Hero~\cite{hero_2013_kronecker} for $N=2$ (see also Beylkin and Mohlenkamp~\cite{Beylkin2002lsr}).}
\begin{Def}\label{separation_rank_def}
The separation rank $\sr(\cdot)$ of a matrix $\m A\in \mb R^{\Pi_n m_n\times \Pi_n p_n}$ is the minimum number $r$ of $N$th-order KS matrices $\m A^k=\bigotimes_{n=1}^N \m A^k_{n}$ such that $\m A=\sum\limits_{k=1}^{r}\bigotimes_{n=1}^N \m A^k_n$, where $\m A^k_n\in \mb R^{m_n \times p_n}$.
\end{Def}
%
%
The KS-DL model corresponds to dictionaries with separation rank $1$. We instead propose the \textit{low separation rank (LSR)} DL model in which the separation rank of the underlying dictionary is relatively small so that $1 \le \mk R_{\v m,\v p}(\D^0)\ll \min\{m,p\}$. 
This generalizes the KS-DL model to a generating dictionary of the form $\D^0=\sum_{k=1}^{r} [\D^k_{N}]^0\otimes [\D^k_{N-1}]^0\otimes \cdots \otimes [\D^k_{1}]^0$, where $r$ is the separation rank of $\D^0$. Consequently, defining $\knr\triangleq \{\D\in \mc D_{m\times p}|\sr(\D)\leq r\}$, the empirical \emph{rank-constrained LSR-DL problem} is
\begin{align}\label{problem:sep_rank_NP}
\min_{\D\in \knr} \fy(\D).
\end{align}
However, the analytical tools at our disposal require the constraint set in \eqref{problem:sep_rank_NP} to be closed, which we show does not hold for $\knr$ when $N \geq 3$ and $r \geq 2$. In that case, we instead analyze \eqref{problem:sep_rank_NP} with $\knr$ replaced by ($i$) closure of $\knr$ and ($ii$) a certain closed subset of $\knr$. We refer the reader to Section~\ref{sec:NP_identif} for further discussion.

In our study of the LSR-DL model (which includes the KS-DL model as a special case), we use a correspondence between KS matrices and rank-1 tensors, stated in Lemma~\ref{lem:rearrangement} below, which allows us to leverage techniques and results in the tensor recovery literature to analyze the LSR-DL problem and develop tractable algorithms. (This correspondence was first exploited in our earlier work~\cite{ghassemi2017stark}.)
\begin{table*}[t!]
\caption{Table of commonly used notation}
\vspace{-0.4\baselineskip}
\label{table:symbol}
\centering
\begin{tabular}{|c|c||c|c|}
\hline
\textbf{Notation} & \textbf{Definition} & \textbf{Notation}  & \textbf{Definition} \\
\hline
$m, p$ &  $\prod_{n=1}^N m_n$, $\prod_{n=1}^N p_n$
& $\v m , \v p$ & $\left(m_n \right)_{n=1}^N$, $\left(p_n\right)_{n=1}^N$
\\
\hline
 $\N_*(\mc A,\epsilon)$
&  Covering number of set $\mc A$ w.r.t. norm $*$ & $\sr(\D)$  & Separation rank of matrix $\D$
 \\
 \hline
$\mc D_{m\times p}$
& Oblique manifold in $\mb R^{m\times p}$ &  $ \ball_{m\times p}$
& 
Euclidean unit sphere in $\mb R^{m\times p}$ \\
 \hline
 $\knrall$ & Set of LSR matrices: $\{\D\in \mb R^{m\times p} | \sr(\D)\leq r\}$ & $\knr$  &
  $\knrall \cap \mc{D}_{m\times p}$
 \\
\hline
 $\kn$ &  $\knr$ with $r=1$: Set of KS matrices on $\mc D_{m\times p}$ & $\kr$ &  $\knr$ with $N=2$
 \\
\hline
$\cknr$ & $ \{\D\in \knr | \|\bigotimes \dkn\|_F\leq c, c>0\}$ & $\knrClosed$ &  Closure of $\knr$\\
%
%
\hline
\multirow{2}{*}{ $\mc C$}
& Compact constraint set in LSR-DL problem:
&  \multirow{2}{*}{$\bro$}  & \multirow{2}{*}{ $\{\D\in \lsr| \|\D-\D^0|_F\leq \rho\}$}\\
& one of $\kn$, $\kr$, $\cknr$, or $\knrClosed$ & & \\
\hline
%
 $\Delta f(\v x_1;\x_2)$ & $f(\x_1)-f(\x_2)$ &$f_{\y}(\D)$ & $\inf_{\x\in \mb R^p} \frac{1}{2} \nor{\y-\D\x}_2^2+\lambda \|\x\|_1$  \\
\hline
 $\fp(\D)$ & $\E_{\y \sim \mc P} ~f_{\y}(\D)$ & $\Delta \fp(\rho)$ & $\inf_{\D\in \partial\bro} \Delta \fp(\D;\D^0)$\\
\hline
 $\fy(\D)$ & $\frac{1}{L}\sum\nolimits_{l=1}^L f_{\y_l}(\D)$ & $\fyreg(\D) $ & $\frac{1}{L}\sum\nolimits_{l=1}^L f_{\y_l}(\D)+ \lambda_1 g_1(\t{D}^{\pi})$\\
\hline
$f^{\mathrm{fac}}_{\y}(\dknset)$ & $\inf_{\x\in \mb R^p} \big\| \y-\big(\sum\nolimits_{k=1}^r\bigotimes\nolimits_{n=1}^N \dkn\big)\x \big\|_2^2  +\lambda\nor{\x}_1$ & $\fyfac(\dknset)$ & $\frac{1}{L}\sum\nolimits_{l=1}^L f^{\mathrm{fac}}_{\y_l}(\dknset)$ \\
\hline

%
%
\end{tabular}
\end{table*}
\begin{figure}[b!]
\centering
\includegraphics[width=0.98\linewidth]{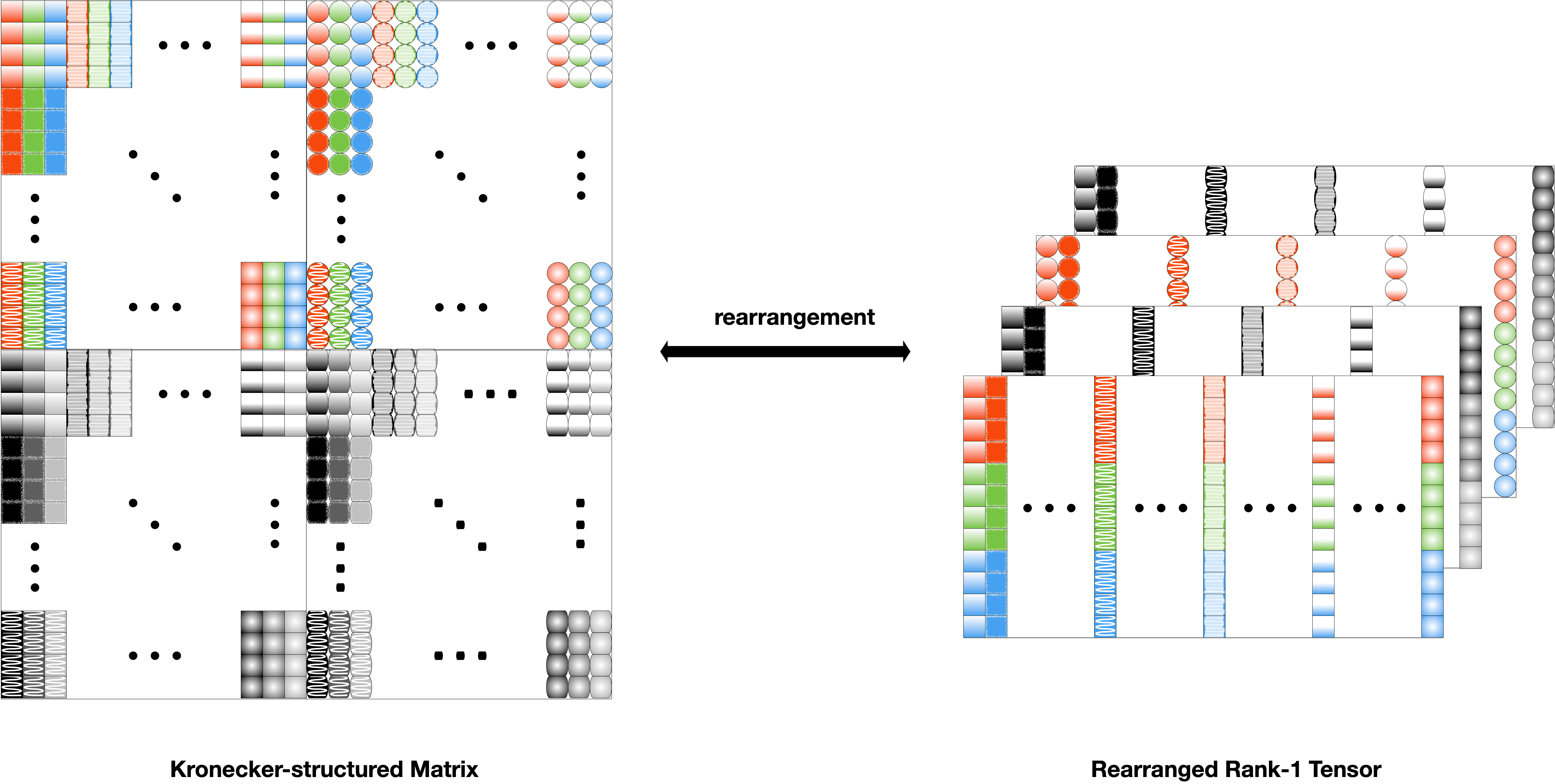}
\caption{Example of rearranging a Kronecker structured matrix ($N=3$) into a third order rank-1 tensor.}
\label{figure_permutation_3}
\end{figure}

\begin{lem}\label{lem:rearrangement}
Any $N$th-order Kronecker-structured matrix $\m A=\m A_1\otimes \m A_2\otimes\cdots\otimes \m A_N$ can be rearranged as a rank-$1$, $N$th-order tensor ${\t A}^{\pi}=\v a_N\circ\cdots\circ\v a_2\circ\v a_1$ with $\v a_n\triangleq \vect(\m A_n)$.
\end{lem}

Figure \ref{figure_permutation_3} provides an example of the rearrangement procedure, which involves finding corresponding indices on the KS matrix and the tensor. A proof of Lemma \ref{lem:rearrangement}, which includes details of the rearrangement strategy, is provided in Appendix~\ref{rearrangement_section}. It follows immediately from Lemma \ref{lem:rearrangement} that if $\D= \sum_{k=1}^r \m D^k_{1} \otimes \cdots \otimes \m D^k_{N} $, then we can rearrange matrix $\D$ into the tensor $\t{D}^{\pi}=\sum_{k=1}^r\v d^k_N \circ \v d^k_{N-1}\circ \cdots \circ \v d^k_1,$ where $\v d_n^k=\vect(\m D_n^k)$. Hence, we have the following equivalence:
\[
\sr(\D)\leq r \Longleftrightarrow
\rank(\Dp)\leq r.
\]
%
%
This correspondence between separation rank and tensor rank highlights a challenge with the LSR-DL problem: finding the rank of a tensor is NP-hard
\cite{haastad1990tensor} and thus so is finding the separation rank of a matrix. This makes Problem~\eqref{problem:sep_rank_NP} in its current form (and its variants) intractable. To overcome this limitation, we introduce two tractable relaxations to the rank-constrained Problem~\eqref{problem:sep_rank_NP} that do not require explicit computation of the tensor rank. The first relaxation 
uses a convex regularization term to implicitly impose low tensor rank structure on $\Dp$, which results in a low separation rank $\D$. The resulting empirical \textit{regularization-based LSR-DL problem} is
%
%
%
\begin{align}\label{problem:regularized}
&\min_{\D\in \mc D_{m\times p}} \fyreg(\D)
\end{align}
with $\fyreg(\D)\triangleq\frac{1}{L}\sum\nolimits_{l=1}^L f_{\y_l}(\D)+ \lambda_1 g_1(\t{D}^{\pi})$, where $f_{\y}(\D)$ is described in \eqref{sparse_coding_general} and $g_1(\Dp)$ is a convex regularizer to enforce low-rank structure on $\Dp$. The second relaxation is a \textit{factorization-based LSR-DL formulation} in which the LSR dictionary is explicitly written in terms of its subdictionaries. The resulting empirical risk minimization problem is
\begin{align}\label{problem:factorized}
\min_{\{\D^k_n\}:~\sum_{k=1}^r\bigotimes_{n=1}^N \dkn\in \mc D_{m\times p}}~ \fyfac\big(\dknset\big),
\end{align}
where $\fyfac(\dknset)\triangleq\frac{1}{L}\sum\nolimits_{l=1}^L f^{\mathrm{fac}}_{\y_l}(\dknset)$ with
\begin{align*}
f^{\mathrm{fac}}_{\y}(\dknset)
\! \triangleq \!
\inf_{\x\in \mb R^p} \Big\| \y \!- \!
\big(\sum\nolimits_{k=1}^r\bigotimes\nolimits_{n=1}^N \dkn\big)\x \Big\|_2^2  +\lambda\nor{\x}_1,
\end{align*}
and the terms $\bigotimes_{n=1}^N \dkn$ are constrained as $\|\bigotimes_{n=1}^N \dkn\|_F \leq c$ for some positive constant $c$ when $N \geq 3$ and $r \geq 2$.
%

In the rest of this paper, we study the problem of identifying the true underlying LSR-DL dictionary by analyzing the LSR-DL Problems~\eqref{problem:sep_rank_NP}--\eqref{problem:factorized} introduced in this section and developing algorithms to solve Problems~\eqref{problem:regularized} and~\eqref{problem:factorized} in both batch and online settings. Note that while Problem~\eqref{problem:sep_rank_NP} (and its variants when $N \geq 3$ and $r \geq 2$) cannot be explicitly solved because of its NP-hardness, identifiability analysis of this problem---provided in Section~\ref{sec:NP_identif}---provides the basis for the analysis of tractable Problems~\eqref{problem:regularized} and~\eqref{problem:factorized}, provided in Section~\ref{sec:tractable_identif}. To improve the readability of our notation-heavy discussions and analysis, we have provided a table of notations (Table~\ref{table:symbol}) for easy access to definitions of the most commonly used notation.
%

\section{Identifiability in the Rank-constrained LSR-DL Problem}\label{sec:NP_identif}
In this section, we derive conditions under which a dictionary $\D^0\in \knr$ is identifiable as a solution to either the separation rank-constrained problem in~\eqref{problem:sep_rank_NP} or a slight variant of~\eqref{problem:sep_rank_NP} when $N \geq 3$ and $r \geq 2$. Specifically, we show that under certain assumptions on the generative model, there is at least one local minimum $\D^*$ of either Problem~\eqref{problem:sep_rank_NP} or one of its variants that is ``close'' to the underlying dictionary $\D^0$. Notwithstanding the fact that no efficient algorithm exists to solve the intractable Problem~\eqref{problem:sep_rank_NP}, this identifiability result is important in that it lays the foundation for the local identifiability results in tractable Problems~\eqref{problem:regularized} and~\eqref{problem:factorized}.

\textbf{Generative Model:}
Let $\D^0\in \knr$ be the underlying dictionary. Each tensor data sample $\t Y \in \mb{R}^{m_1 \times m_2 \times \cdots \times m_N}$ in its vectorized form is \textit{independently} generated using a linear combination of $s\ll p$ atoms of dictionary $\D^0$ with added noise: $\y \triangleq \vect(\t Y)=\D^0\x^0+\bm \epsilon$, where $\nor{\x^0}_0\leq s$. Specifically, $s$ atoms of $\D^0$ are selected uniformly at random, defining the support $\mc J\subset [p]$. Then, we draw a random sparse coefficient vector $\x^0\in \mb R^{p}$ supported on $\mc J$. We state further assumptions on our model similar to prior works~\cite{gribonval2014sparse,shakeri2018achieve}.

%
%
%

\begin{assum}[{Coefficient Distribution}]\label{assumptions:coefficient} Consider a random variable $x \in \mb R$ and positive constants $M_x$ and $\underline x$. Define $\s^0\triangleq \sgn(\x^0)$. We assume: \textbf{i)} $\E\big\{\x_{\mc J}^0 [\x_{\mc J}^0]^T|\mc J\big\}=\E\{x^2\}\cdot \m I_s$,
\textbf{ii)} $\E\big\{\s_{\mc J}^0 [\s_{\mc J}^0]^T|\mc J\big\}= \m I_s $,
\textbf{iii)} $\E\big\{\s_{\mc J}^0 [\x_{\mc J}^0]^T|\mc J\big\}=\E\{|x|\}\cdot \m I_s$, and
\textbf{iv)} $\nor{\x^0}_2\leq M_{x}$ and $\min\limits_{j\in \mc J} |\x_j^0| \geq \underline x$ almost surely.
\end{assum}

%
%

\begin{assum}[Noise Distribution]\label{assumptions:noise}
Consider a random variable $\epsilon \in \mb R$ and positive constant $M_{\epsilon}$. We assume:
\textbf{i)} $\E\big\{\bm \epsilon \bm \epsilon^T|\mc J\big\}=\E\{\epsilon^2\}\cdot \m I_m$,
\textbf{ii)} $\E\big\{ \x^0 \bm \epsilon^T|\mc J\big\}=\E\big\{ \s^0 \bm \epsilon^T|\mc J\big\}= 0$, and
\textbf{iii)} $\nor{\bm \epsilon}_2\leq M_{\epsilon}$ almost surely.
\end{assum}
Note that Assumptions \ref{assumptions:coefficient}-iv and \ref{assumptions:noise}-iii imply the magnitude of $\y$ is bounded: $\|\y\|_2\leq M_y$. Next, we define positive parameters $\bar\lambda\triangleq \frac{\lambda}{ \E\{|x|\}}$, $C_{\min}\triangleq 24 \frac{\E\{|x|\}^2}{\E\{x^2\}}\pr{\nor{\D^0}_2+1}^2 \frac{s}{p} \nor{[\D^0]^T\D^0-\m I}_F$, and $C_{\max}\triangleq\frac{2\E \{|x|\} }{7 M_{x}}\pr{1-2\mu_s(\D^0)}$ for ease of notation. We use the following assumption, similar to Gribonval et al.~\cite [Thm. 1]{gribonval2014sparse}.
\begin{assum}\label{assumptions:parameter}
Assume $C_{\min}\leq C_{\max}$, $\lambda \leq \underline x/4$, $s\leq \frac{p}{16\pr{\nor{\D^0}_2+1}^2}$, $\mu_s(\D^0)\leq 1/4$, and the noise is relatively small in the sense that $\frac{M_{\epsilon}}{M_{x}}<\frac{7}{2}\pr{C_{\max}-C_{\min}}\bar\lambda$.
\end{assum}

\textbf{Our Approach:}
In our analysis of the separation rank-constrained LSR-DL problem, we will alternate between four different constraint sets that are related to our dictionary class $\knr$, namely, $\kr$, $\kn$, the closure $\knrClosed\triangleq \cl(\knr)$ of $\knr$ under the Frobenius norm, and a closed subset of $\knr$, defined as $\cknr\triangleq \{\D\in \knr | \|\bigotimes \dkn\|_F\leq c, c>0\}$. We often use the generic notation $\lsr$ for the constraint set when our discussion is applicable to more than one of these sets.

We want to find conditions that imply the existence of a local minimum of $\min_{\D\in \lsr} \fy(\D)$ within a ball of radius $\rho$ around the true dictionary $\D^0\in \knr$:
\begin{align}\label{bro}
\bro\triangleq\{\D\in \lsr| \nor{\D-\D^0}_F\leq \rho\}
\end{align}
for some small $\rho>0$. To this end, we first show that the expected risk function $\fp(\D)$ in \eqref{problem:expected_risk_minimization_general} has a local minimum in $\bro$ for the LSR-DL constraint set $\lsr$.

To show that a local minimum of $\fp: \lsr \mapsto \mb R$ exists in $\bro$, we need to show that $\fp(\D)$ attains its minimum over $\bro$ in the interior of $\bro$.\footnote{ Having a minimum $\D^*$ on the boundary is not sufficient. 
If the minimizer of $\fp(\D)$ over $\bro$ is on the boundary of $\bro$, the value of $\fp(\D)$ in the neighborhood of $\D^*$ outside $\bro$ can be smaller than $\fp(\D^*)$; therefore, $\D^*$ is not necessarily a local minimum of $\D\in \mc C \mapsto \fp(\D)$.} We show this in two stages. First, we use the Weierstrass Extreme Value Theorem~\cite{rudin1964principles}, which dictates that the continuous function $\fp(\D)$ attains a minimum in (or on the boundary of) $\bro$ as long as $\bro$ is a compact set. 
Therefore, we first investigate compactness of $\bro$ in Section \ref{subsec:compactness}.
Second, in order to be certain that the minimizer of $\fp(\D)$ over $\bro$ is a local minimum of $\D\in \mc C \mapsto \fp(\D)$, we show that $\fp(\D)$ cannot obtain its minimum over $\mc B_{\rho}$ on the boundary of $\mc B_{\rho}$, denoted by $\partial \bro$.
%
To this end, in Section \ref{subsec:asymptotic} we derive conditions that if $\partial \bro$ is nonempty then we have\footnote{If the boundary is empty, it is trivial that the infimum is attained in the interior of the set.}
\begin{align}\label{sphere_condition_asymptotic}
\Delta \fp(\rho)\triangleq\inf_{\D\in \partial\bro} \Delta \fp(\D;\D^0)>0,
\end{align}
which implies $\fp(\D)$ cannot achieve its minimum on $\partial \bro$.

Finally, in Section \ref{subsec:finite} we use concentration of measure inequalities 
to relate $\fy(\D)$ in~\eqref{problem:empirical_risk_minimization_general} to $\fp(\D)$ and find the number of samples needed to guarantee (with high probability) that $\fy(\D)$ also has a local minimum in the interior of $\bro$.

\subsection{Compactness of the Constraint Sets}\label{subsec:compactness}
When the constraint set $\lsr$ is a compact subset of the Euclidean space $\mb R^{m\times p}$, the subset $\bro$ is also compact. 
Thus, we first investigate the compactness of the constraint set $\knr$. Since $\knr$ is a bounded set, according to the Heine-Borel Theorem~\cite{rudin1964principles}, it is a compact subset of $\mb R^{m\times p}$ if and only if it is closed. Also, $\knr$ can be written as the intersection of $\knrall \triangleq \{\D\in \mb R^{m\times p} | \sr(\D)\leq r\}$ and the oblique manifold $\mc D$. In order for $\knr = \knrall \cap \mc{D}$ to be closed, it suffices to show that $\knrall$ and $\mc D$ are closed. It is trivial to show $\mc D$ is closed; hence, we focus on whether $\knrall$ is closed. 

In the following, we use the facts that the constraint $\sr(\D)\leq r$ 
is equivalent to $\rank(\Dp)\leq r$ and that the rearrangement mapping 
that sends $\D$ to $\Dp$ preserves topological properties of sets such as the distances between the set elements under the Frobenius norm. These facts allow us to translate the topological properties of tensor sets into properties of the structured matrices that we study here.

\textit{Remark.} Proofs of Lemmas~\ref{lem:closedness_LSR},~\ref{lem:unbounded_sequence},~\ref{lem:covering_number_2nd},  \ref{lem:covering_number_general}, and~\ref{lem:neighborhoods_relation_D_to_Dkn}  are provided in Appendix \ref{proofs_section}.
\begin{lem}\label{lem:closedness_LSR}
Let $N \geq  3$ and $r \geq 2$. 
Then, the set $\knrall$ is not closed. However, the set of KS matrices ${\mc L}_{\v m,\v p}^{N,1}$ and the set ${\mc L}_{\v m,\v p}^{2,r}$ are closed.
\end{lem}
To illustrate the non-closedness of $\knrall$ for $N \geq  3$ and $r \geq 2$ and motivate the use of the sets $\knrClosed$ and $\cknr$ in lieu of $\knr$, we provide an example. 

\textit{Example.}
Consider the sequence
%
$\D_t :=t \left(\m A_1+\frac{1}{t}\m B_1 \right)\otimes \left(\m A_2+\frac{1}{t}\m B_2 \right)\otimes \left(\m A_3+\frac{1}{t}\m B_3 \right))
%
-t \m A_1\otimes \m A_2\otimes \m A_3$
%
where $\m A_i, \m B_i\in \mb R^{m_i\times p_i}$ are linearly independent pairs. 
Here, $\mk R^3_{\m m,\m p}(\D_t)\leq 2$ for any $t$. The limit point of this sequence is $\lim_{t\rightarrow\infty} \D_t
=\m A_1 \otimes \m A_2 \otimes \m B_3+\m A_1 \otimes \m B_2 \otimes \m A_3+\m B_1 \otimes \m A_2 \otimes \m B_3$,
%
%
%
which is a separation-rank-$3$ matrix. Hence, the set $\mc L^{3,2}_{\v m,\v p}$ is not closed.

The non-closedness of $\knrall$ means there exist sequences in $\knrall$ whose limit points are not in the set. Two possible solutions to circumvent this issue include: ($i$) use the closure of $\knrall$ as the constraint set, and ($ii$) eliminate such sequences from $\knrall$. We discuss each solution in detail below.
%

\paragraph{Adding the limit points} We denote the closure of $\knrall$ by $\knrallClosed\triangleq\cl (\knrall)$. 
By slightly relaxing the constraint set in \eqref{problem:sep_rank_NP} to $\knrallClosed \cap \mc D$, we can instead solve the following:
\begin{align}\label{problem:sep_rank_closure_NP}
&\min_{\D\in \knrClosed} \fy(\D),
\end{align}
%
where $\knrClosed=\knrallClosed\cap \mc D$. Note that ($i$) a solution to \eqref{problem:sep_rank_NP} is a solution to \eqref{problem:sep_rank_closure_NP} and ($ii$) a solution to \eqref{problem:sep_rank_closure_NP} is either a solution to \eqref{problem:sep_rank_NP} or is arbitrarily close to a member of $\knr$.\footnote{The first argument holds since if $\fy(\D^*)\leq \fy(\D)$ for all $\D\in \knr$, by continuity it also holds for all $\D\in \knrClosed$.
The second argument is trivial.}

\paragraph{Eliminating the problematic sequences}\label{general_case_close_eliminate} In order to exclude the sequences $\D_t\rightarrow \D$ such that $\D_t \in \knrall$ for all $t$ and $\D \notin \knrall$, we first need to characterize them.
\begin{lem}\label{lem:unbounded_sequence}
Assume $\D_t\rightarrow \D$ where $\sr(\D_t)\leq r$ and $\sr(\D)>r$. We can write $\D_t=\sum_{k=1}^r \lambda_{t}^k \bigotimes_{n=1}^N [\dkn]_t$ where $\nor{[\dkn]_t}_F=1$. Then, $\max_k |\lambda^k_t|\rightarrow \infty$ as $t\rightarrow \infty$. In fact,  at least two of the coefficient sequences $\lambda^k_t$ are unbounded.
\end{lem}
The following corollary of Lemma \ref{lem:unbounded_sequence} suggests that one can exclude the problematic sequences from $\knrall$ by bounding the norm of individual KS (separation-rank-$1$) terms.
\begin{cor}\label{cor:bounded_closedness}
Consider the set $\knrall$ whose members can be written as $\D=\sum_{k=1}^r \bigotimes_{n=1}^N \dkn$ such that $\dkn\in \mb R^{m_n\times p_n}$.
Then, for any $c>0$ the set
$
\cknrall=\big\{\D\in \knrall|\big\|\bigotimes \dkn\big\|_F\leq c\big\}
$
is closed.
\end{cor}
%


We have now shown that the sets $\kr$, $\kn\triangleq\mc K^{N,1}_{\mathbf{m}, \mathbf{p}}$, $\cknr=\cknrall\cap \mc D$, and $\knrClosed=\knrallClosed\cap \mc D$  are compact subsets of $\mb R^{m\times p}$. 


\subsection{Asymptotic Analysis for Dictionary Identifiability}\label{subsec:asymptotic}
Now that we have discussed the compactness of the relevant constraint sets, we are ready to show that the minimum of $f_{\y}(\D)$ over $\bro$, defined in \eqref{bro}, is not attained on $\partial \bro$. This will complete our proof of existence of a local minimum of $\fp(\D)$ in $\bro$. In our proof, we make use of a result in Gribonval et al. \cite{gribonval2014sparse}, presented here in Lemma \ref{lem:gribonval_boundary}.

\begin{lem}[Theorem 1 in Gribonval et al. \cite{gribonval2014sparse}]\label{lem:gribonval_boundary}
Consider the statistical DL Problem \eqref{problem:expected_risk_minimization_general} with constraint set $\mc D$. Suppose the generating dictionary $\D^0\in \mc D$ and Assumptions \ref{assumptions:coefficient}--\ref{assumptions:parameter} hold. Then, for any $\rho$ such that $\bar\lambda C_{\min}<\rho\leq \bar\lambda C_{\max}$ and $\frac{M_{\epsilon}}{M_{x}}<\frac{7}{2}(\bar\lambda C_{max}-\rho)$, we have
\begin{align}\label{lower_bound_delta_f_p}
\Delta f_{\mc P}(\rho)\geq \frac{\E\{x^2\}}{8}\cdot \frac{s}{p}\cdot \rho \pr{\rho-\bar\lambda C_{\min}}>0.
\end{align}
for all $\D\in \mc D$ such that $\|\D-\D^0\|_F = \rho$.
\end{lem}
Interested readers can find the detailed proof of Lemma \ref{lem:gribonval_boundary} in Gribonval et al.~\cite{gribonval2014sparse}. The following theorem states our first identifiability result for the LSR-DL model.

\begin{thm}\label{thm:asymptotic_nonconvex}
Consider the statistical DL Problem~\eqref{problem:expected_risk_minimization_general} with constraint set $\mc C$ being either $\kr$, $\kn$, $\cknr$ or $\knrClosed$. Suppose the generating dictionary $\D^0\in \mc C$ and Assumptions~\ref{assumptions:coefficient}--\ref{assumptions:parameter} hold. Then, for any $\rho$ such that $\bar\lambda C_{\min}<\rho< \bar\lambda C_{\max}$ and $\frac{M_{\epsilon}}{M_{x}}<\frac{7}{2}(\bar\lambda C_{max}-\rho)$, the function $\D\in \mc C \mapsto \fp(\D)$ has a local minimum $\D^*$ such that $\nor{\D^*-\D^0}_F < \rho$.
\end{thm}
\begin{proof}
Since $\fp(\D)$ is a continuous function and the ball $\bro=\{\D\in \lsr| \nor{\D-\D^0}_F\leq \rho\}$ is compact, 
the function $\D\in \bro \mapsto \fp(\D)$ attains its infimum at a point in the ball. If this minimum is attained in the interior of $\bro$ then it is a local minimum of $\D\in \mc C \mapsto \fp(\D)$. Therefore, a key ingredient of the proof is showing that $\fp(\D)>\fp(\D^0)$ for all  $\D\in \partial \bro$ if $\partial\bro$ is nonempty.
Lemma \ref{lem:gribonval_boundary} states the conditions under which $\fp(\D)>\fp(\D^0)$ on $\partial {\mc S}_{\rho}$, where
$
{\mc S}_{\rho}\triangleq\big\{\D\in \mc D ~\big|~ \nor{\D-\D^0}_F\leq\rho \big\}$.
%

Since $\partial \bro\subset \partial{\mc S}_{\rho}$, the result of Lemma \ref{lem:gribonval_boundary} can be used for our problem as well, i.e. for any $\D\in\partial \bro$, we have $\fp(\D)>\fp(\D^0)$, when $C_{\min}\bar\lambda<\rho<C_{\max}\bar\lambda$.
It follows from this result together with the existence of the infimum of $\fp(\D): \bro\mapsto \mb R$ in $\bro$ that Problem~\eqref{problem:expected_risk_minimization_general} has a local minimum within a ball of radius $\rho$ around the true dictionary $\D^0$.
\end{proof}

In Theorem~\ref{thm:asymptotic_nonconvex}, we guarantee that the true dictionary, $\D^0$, is identifiable as a solution to the statistical rank-constrained LSR-DL problem \eqref{problem:expected_risk_minimization_general}. Next, we take advantage of concentration of measure inequalities 
that relate the empirical objective 
in~\eqref{problem:empirical_risk_minimization_general} to the statistical objective studied in Theorem~\ref{thm:asymptotic_nonconvex} to find the number of samples needed to ensure $\D^0$ is also identifiable via the empirical rank-constrained LSR-DL problem \eqref{problem:expected_risk_minimization_general}.

\vspace{-0.5\baselineskip}
\subsection{Sample Complexity for Dictionary Identifiability}\label{subsec:finite}
We now derive the number of samples required to guarantee, with high probability, that $\fy:\mc C\mapsto \mb R$ has a local minimum at a point ``close'' to $\D^0$ when the constraint set $\mc C$ is either $\kr$, $\kn$, or $\cknr$ for $N \geq 3$ and $r \geq 2$. 
%
First, we use concentration of measure inequalities based on the covering number of the dictionary class $\mc C \subset \knr$ 
to show that the empirical loss $\fy(\D)$ uniformly converges to its expectation $\fp(\D)$ with high probability. This is formalized below.
\begin{lem}[Theorem 1 and Lemma 11, Gribonval et al.~\cite{Gribonval2015sample}]\label{lem:gribonval_eta_bound}
Consider the empirical DL Problem~\eqref{problem:empirical_risk_minimization_general} and suppose Assumptions~\ref{assumptions:coefficient} and~\ref{assumptions:noise} are satisfied. For any $u\geq 0$ and constants $c_1\geq M_{y}^2/\sqrt{8}$ and $c_2\geq\max(1, \log c_0\sqrt{8}M_{y})$, with probability at least $1-2e^{-u}$ we have
\begin{align}\label{deviation_bound}
\sup_{\D\in \mc C} |F_{\Y}(\D)-f_{\mc P}(\D)| \leq 3c_1\sqrt{\frac{c_2 \nu \log L}{L}}+ c_1\sqrt{\frac{c_2 \nu +u}{L}},
\end{align}
where $\nu$ is such that $\N_{\mxc}(\mc C,\epsilon)=\pr{\frac{c_0}{\epsilon}}^{\nu}$.
\end{lem}
Define $\eta_L \triangleq 3c_1\sqrt{\frac{c_2 \nu \log L}{L}}+ c_1\sqrt{\frac{c_2 \nu +u}{L}}$.
%
%
It follows from \eqref{deviation_bound} that with high probability (w.h.p.),
\begin{align}\label{relation_f_p_F_Y}
\Delta \fy(\D;\D^0)\geq \Delta \fp(\D;\D^0)-2\eta_L,
\end{align}
for all $\D\in \mc C$. Therefore, when 
$\eta_L<  \Delta \fp(\D;\D^0)/2$  for all $\D\in \partial \bro$,
we have $ \Delta \fy(\D;\D^0)>0$ for all $\D\in \partial \bro$. In this case, we can use similar arguments as in the asymptotic analysis to show that $\fy:\mc C \rightarrow \mb R$ has a local minimum at a point in the interior of $\bro$. Hence, our focus in this section is on finding the sample complexity $L$ required to guarantee that $\eta_L\leq  \Delta \fp(\rho)/2$ w.h.p. We begin with characterization of covering numbers of the three constraint sets, which may also be of independent interest to some readers.

\textbf{Covering Numbers:}
The covering number of the set $\kn$ 
with respect to the norm $\|\cdot\|_{\mxc}$ is known in the literature to be upper bounded as follows~\cite{Gribonval2015sample}:
\begin{align}\label{covering_number_KS}
  \N_{\mxc}(\kn,\epsilon)\leq (3/\epsilon)^{\sum_{i=1}^N m_i p_i}.
\end{align}
We now turn to finding the covering numbers of LSR sets $\kr$ and $\cknr$. The following lemma establishes a bound on covering number of $\kr$, which depends on the separation rank $r$ exponentially.
%
\begin{lem}\label{lem:covering_number_2nd}
The covering number of the set $\kr$ 
 with respect to the norm $\|\cdot\|_{\mxc}$ is upper bounded as follows:
\[
\N_{\mxc}(\kr,\epsilon)\leq (9p/\epsilon)^{r(m_1p_1+m_2p_2+1)}.
\]
\end{lem}
Next, we obtain an upper bound on the covering number of $\cknr$ for a given constant $c$. 



\begin{lem}\label{lem:covering_number_general}
The covering number of the set $\cknr$ 
with respect to the max-column norm $\|\cdot\|_{\mxc}$ is bounded as follows:
\[
\mc N_{\mxc}(\cknr,\epsilon)\leq (3rc/\epsilon)^{r\sum_{i=1}^N m_i p_i}.
\]
\end{lem}
%
%
We can now find the sample complexity of the LSR-DL Problem~\eqref{problem:empirical_risk_minimization_general} by plugging in the values of $\nu$ and $c_0$ in Lemma~\ref{lem:gribonval_eta_bound}. 
\begin{thm}\label{thm:finite_nonconvex}
Consider the empirical LSR dictionary learning Problem~\eqref{problem:empirical_risk_minimization_general} with constraint set $\mc C$ being $\kr$, $\kn$, or $\cknr$. Fix any $u>0$. Suppose the generating dictionary $\D^0\in \mc C$ and 
Assumptions \ref{assumptions:coefficient}--\ref{assumptions:parameter} are satisfied. Assume $\bar\lambda C_{\min}<\rho< \bar\lambda C_{\max}$ and $\frac{M_{\epsilon}}{M_{x}}<\frac{7}{2}(\bar\lambda C_{max}-\rho)$. Define a constant $\nu$ that depends on the dictionary class:
\begin{itemize}
\item $\nu=\sum_{i=1}^N m_i p_i$ and $c_0=3$ when $\mc C=\kn$,
\item $\nu={2r(m_1p_1+m_2p_2+1)}$ and $c_0=9p$ when $\mc C=\kr$,
\item $\nu={r\sum_{i=1}^N m_i p_i}$ and $c_0=rc$ when $\mc C=\cknr$.
\end{itemize}
Then, given a number of samples $L$ satisfying
\begin{align}\label{sample_complexity_general}
\frac{L}{\log L}\geq C p^2 \pr{ \nu \log c_0 + u}\frac{M_y^4}{\pr{\rho\pr{\rho-\bar\lambda C_{\min}}s\E \{x^2\}}^2}
\end{align}
where $C$ is a constant,  with probability no less than $1-e^{-u}$, 
the empirical risk objective function $\D\in \mc C \mapsto \fy(\D)$ has a local minimizer $\D^*$ such that $\nor{\D^*-\D^0}_F< \rho$. 
\end{thm}

\begin{proof}
We take a similar approach to the proof of Theorem \ref{thm:asymptotic_nonconvex}. Due to compactness of the ball $\bro=\{\D\in \lsr| \nor{\D-\D^0}_F\leq \rho\}$ and continuity of $\fy(\D)$, it follows 
that $\D\in \bro \mapsto \fy(\D)$ attains its minimum at a point in $\bro$.  It remains to show that $\Delta \fy(\D;\D^0)>0$ for all $\D\in \partial \mc \bro$ which implies existence of a local minimizer of $\fy: \mc C \rightarrow \mb R$ at $\D^*$ such that $\nor{\D^*-\D^0}_F< \rho$.

Inequality \eqref{relation_f_p_F_Y} shows that it suffices to set $\eta_L\leq \Delta \fp(\D;\D^0)/2$ to have $\Delta \fy(\D;\D^0)>0$.
From Lemma \ref{lem:gribonval_eta_bound} we know $\eta_L\geq 3c_1\sqrt{\frac{c_2 \nu \log L}{L}}+ c_1\sqrt{\frac{c_2 \nu +u}{L}}$. Therefore,
using the lower bound \eqref{lower_bound_delta_f_p} on $\Delta \fp(\rho)$ we have with probability at least $1-e^{-u}$
\begin{align*}
3c_1\sqrt{\frac{c_2 \nu \log L}{L}}+ c_1\sqrt{\frac{c_2 \nu +u}{L}}\leq \frac{\E \{x^2\}}{16}\cdot \frac{s}{p}\cdot \rho \pr{\rho-\bar\lambda C_{\min}}
\end{align*}
with $c_1\geq M_{y}^2/\sqrt{8}$ and $c_2\geq\max(1, \log c_0\sqrt{8}M_{y})$\footnote{Under the conditions of this theorem, $M_y\leq \sqrt{1+\delta_s(\D^0)}M_x + M_{\epsilon}$, where $\delta_s(\D^0)$ denotes the RIP constant of $\D^0$.}. Rearranging, we get
\begin{align}
\frac{L}{\log L}\geq c_1^2 \pr{\frac{3\sqrt{c_2 \nu}+ \sqrt{c_2\nu +u}}{ \rho\pr{\rho-\bar\lambda C_{\min}}}}^2 \pr{\frac{16}{\E \{x^2\}}\cdot \frac{p}{s}}^2 .
\end{align}
Setting $c_1\geq M_y^2/\sqrt{8}$ 
and $c_2=c_3 \log c_0 \geq\max(1, \log c_0\sqrt{8}M_{y})$ we get the lower bound
\begin{align*}
\frac{L}{\log L}\geq C p^2 \pr{ \nu\log c_0 + u}\pr{\frac{M_y^2}{\rho\pr{\rho-\bar\lambda C_{\min}}s\E \{x^2\}}}^2
\end{align*}
with probability at least $1-e^{-u}$. 
Given that the number of samples satisfies \eqref{sample_complexity_general} for $\bar\lambda C_{\min}<\rho< \bar\lambda C_{\max}$, with high probability $\Delta \fy>0$ for any $\D\in \partial \bro$. Therefore, it follows from the existence of the infimum of $\D\in \bro \mapsto \fy(\D)$ in $\bro$ that $\D\in \mc C \mapsto \fy(\D)$ 
has a local minimum at a point within a ball of radius $\rho$ around the true dictionary $\D^0$.
\end{proof}

The $\Omega\big(r(\sum_n m_n p_n) p^2 \rho^{-2}\big)$ sample complexity upper bound we obtain here for rank-constrained LSR-DL is a reduction compared to the $\Omega(mp^3\rho^{-2})$ sample complexity of standard DL~\cite{gribonval2014sparse}.
However, the minimax lower bound scaling of $\Omega(p\sum_n m_np_n \rho^{-2})$  for KS-DL~\cite{shakeri2016minimax} ($r=1$) indicates an $O(p)$ gap with our sample complexity upper bound. This gap could be due to looseness in the lower bound, our upper bound, or both. We leave an investigation of this and possible tightening of the bound(s) to future work.



\section{Identifiability in the Tractable LSR-DL Problems}\label{sec:tractable_identif}
In Section~\ref{sec:problem_statement}, we introduced two tractable relaxations to the rank-constrained LSR-DL problem: a regularized problem \eqref{problem:regularized} with a convex regularization term and a factorization-based problem \eqref{problem:factorized} in which the dictionary is written in terms of its subdictionaries. 
We now provide results on the local identifiability of the true dictionary $\D^0$ in these problems, i.e., we find conditions under which at least one local minimizer of these problems is located near the true dictionary $\D^0$. Such local identifiability result implies that any DL algorithm that converges to a local minimum of these problems can recover $\D^0$ up to a small error if it is initialized close enough to $\D^0$.

\subsection{Regularization-based LSR Dictionary Learning}\label{regularization_section}


The first tractable LSR-DL problem that we study is the regularized problem~\eqref{problem:regularized}. Exploiting the relation between $\sr(\D)$ and $\rank(\Dp)$, the LSR structure is enforced on the dictionary by a convex regularizer that imposes low tensor rank structure on $\Dp$. 
The regularizer that we use here is a commonly used convex proxy for the tensor rank function, the \textit{sum-trace-norm}~\cite{Wimalawarne_2014}, which is defined as the average of the trace (nuclear) norms of the \textit{unfoldings} of the tensor:
$
\strnorm{\t A}\triangleq
\sum\nolimits_{n=1}^{N} \trnorm{\m A^{(n)}}
$. 
%
%
%
%
%


The first question we address is whether the reference dictionary that generates the observations $\{\t Y_l\}_{l=1}^L$ is identifiable via Problem~\eqref{problem:regularized}. Our local identifiability result here is limited to when $\D^0\in \kn$, i.e. the true dictionary is KS. For such $\D^0$, we show that there is at least one local minimizer $\D^*$ of $\fyreg(\D)$ under Assumptions~\ref{assumptions:coefficient}--\ref{assumptions:parameter} that is close to $\D^0$. 
\begin{thm}\label{thm:regularized_theorem}
Consider the regularized LSR-DL problem~\eqref{problem:regularized}. 
Suppose that the generating dictionary $\D^0\in \kn$ and Assumptions~\ref{assumptions:coefficient}--\ref{assumptions:parameter} are satisfied. Moreover, let $\bar\lambda C_{\min}<\rho\leq \bar\lambda C_{\max}$ and $\frac{M_{\epsilon}}{M_{x}}<\frac{7}{2}(\bar\lambda C_{max}-\rho)$.  Then, the expected risk function $\D\in \mc D \mapsto \E[\fyreg(\D)]$ has a local minimizer $\D^*$ such that $\nor{\D^*-\D^0}_F\leq \rho$.

Moreover, given $L$ samples such that 
\begin{align}\label{sample_complexity_convex}
L>C_0 p^2 (mp+u)  \left(\frac{M_x^2}{\E x^2}\cdot\frac{\frac{M_{\epsilon}}{M_x}+\rho+(\frac{M_{\epsilon}}{M_x}+\rho)^2}{\rho-C_{\min}\bar\lambda} \right)^2,
\end{align}
where $u$ and $C_0$ are positive constants, then, we have with probability no less than $1-e^{-u}$ that the empirical risk function $\D\in \mc D \mapsto \fyreg(\D)$ has a local minimum at $\D^*$ such that $\nor{\D^*-\D^0}_F< \rho$.
\end{thm}

\begin{proof}
Consider the ball $\bro=\{\D\in \mc D| \nor{\D-\D^0}_F\leq \rho\}$. 
Compactness of $\bro=\{\D\in \mc C|\nor{\D-\D^0}_F \leq \rho\}$ and continuity of $\fyreg(\D)$ guarantee that $\D\in \bro \mapsto \fyreg(\D)$ attains its minimum at a point in $\bro$. Similarly, $\D\in \bro \mapsto \fpreg(\D)\triangleq \E[\fyreg]$ reaches its minimum at a point in $\bro$. We now need to show in either case the minimum is not attained on the boundary of $\bro$. To this end, we show in the following that $\Delta \fyreg(\D;\D^0)>0$ and $\Delta \fpreg(\D;\D^0)>0$ for any $\D\in \partial\bro$. 

Incorporation of the trace-norm regularization term in \eqref{problem:regularized} within the objective in \eqref{problem:empirical_risk_minimization_general} introduces a factor $\strnorm{\Dp} -\strnorm{[\t D^0]^{\pi}}=\sum_{n=1}^N\left(\trnorm{\D^{(n)}}-\trnorm{[\D^0]^{(n)}}\right)$
to 
$\Delta \fp(\D;\D^0)$ and $\Delta \fy(\D;\D^0)$. We know from Lemma \ref{lem:rearrangement} that when the true dictionary is a KS matrix ($\D^0\in \kn$), its rearrangement tensor $[\t D^0]^{\pi}$  is a rank-$1$ tensor and therefore all unfoldings $[\D^0]^{(n)}$ of $[\t D^0]^{\pi}$ are rank-$1$ matrices. This implies $\|[\D^0]^{(n)}\|_{\mathrm{tr}}=\|[\D^0]^{(n)}\|_F$.
Moreover, for all $\D\in \mc D_{m\times p}$ we have  
$\big\|\D^{(n)}\big\|_F=\big\|[\D^0]^{(n)}\big\|_F=\sqrt{p}$. 
Therefore,
\begin{align*}
\big\|\D^{(n)}\big\|_{\mathrm{tr}}-\big\|[\D^0]^{(n)}\big\|_{\mathrm{tr}}
&=\sum\nolimits_{k=1}^{r_n} \sigma_k(\D^{(n)})-\sqrt{p}\nonumber\\
&\geq \sqrt{\sum\nolimits_{k=1}^{r_n} \sigma_k^2(\D^{(n)})}-\sqrt{p}=0,
%
\end{align*}
where $r_n\triangleq \rank(\D^{(n)})$ and $\sigma_k(\D^{(n)})$ denotes the $k$-th singular value of $\D^{(n)}$. Therefore, we conclude that $\Delta \fyreg(\D;\D^0)\geq\Delta \fy(\D;\D^0)$ and $\Delta \fpreg(\D;\D^0)\geq\Delta \fp(\D;\D^0)$ for any $\D\in \mc D$. 
According to Lemma \ref{lem:gribonval_boundary}, $\Delta f_{\mc P}(\D;\D^0)>0$ for all $\D$ on the boundary of the ball $\bro$.
Furthermore, under the assumptions of the current theorem, given a number of samples satisfying \eqref{sample_complexity_convex}, 
Gribonval et al.~\cite{gribonval2014sparse}~show that the empirical difference $\Delta F_{\Y}(\D;\D^0)>0$ for all $\D$ on the boundary of ${\mc S}_{\rho}=\big\{\D\in \mc D ~\big|~ \nor{\D-\D^0}_F\leq\rho \big\}$, and therefore on the boundary of $\bro\subseteq \mc S_{\rho}$, with probability at least $1-e^{-u}$. 
Therefore, for both $\fpreg(\D)$ and $\fyreg(\D)$, the minimum is attained in the interior of $\bro$ and not on its boundary.
\end{proof}
We discuss the implications of Theorem \ref{thm:regularized_theorem} in Section \ref{subsec:discussion}.



\subsection{Factorization-based LSR Dictionary Learning}\label{factorization_section}

We now shift our focus to Problem~\eqref{problem:factorized}, which expands $\D$ as $\sum_{k=1}^r\bigotimes \dkn$ and optimizes over the individual subdictionaries, and show that there is at least one local minimum $\{[\dkn]^*\}$ of the factorization-based LSR-DL Problem~\eqref{problem:factorized} such that $\sum\bigotimes [\dkn]^*$ is close to the underlying dictionary $\D^0$.
%
%
%
Our strategy here is to establish a connection between the local minima of \eqref{problem:factorized} and those of \eqref{problem:empirical_risk_minimization_general}. Specifically, we show that when the dictionary class in \eqref{problem:sep_rank_NP} matches that of \eqref{problem:factorized}, for every local minimum $\widehat\D$ of  \eqref{problem:empirical_risk_minimization_general}, there exists a local minimum  $\{\widehat\D^k_n\}$ of \eqref{problem:factorized} such that $\widehat \D=\sum\bigotimes \widehat\D^k_n$. Furthermore, we use the result of Theorems \ref{thm:asymptotic_nonconvex}~and~\ref{thm:finite_nonconvex} that 
there exists a local minimum $\D^*$ of Problem \eqref{problem:empirical_risk_minimization_general} within a small ball around $\D^0$. 
It follows from these facts that under the generating model considered here, 
a local minimum $\{[\dkn]^*\}$ of \eqref{problem:factorized} is such that $\sum\bigotimes [\dkn]^*$ is close to $\D^0$. 

We begin with a bound on the distance between LSR matrices when the tuples of their factor matrices are $\epsilon$-close.

\begin{lem}\label{lem:neighborhoods_relation_D_to_Dkn}
For any two tuples $(\m A^k_n)$ and $(\m B^k_n)$ such that $\m A^k_n, \m B^k_n\in \alpha \ball_{m_n\times p_n}$ for all $n\in[N]$ and $k\in [r]$, if the distance $\big\|(\m A^k_n)-(\m B^k_n)\big\|_F\leq \epsilon$ then $\big\|\sum\nolimits_{k=1}^r\bigotimes \m A^k_n-\sum_{k=1}^r\bigotimes \m B^k_n\big\|_F\leq \alpha^{N-1}\sqrt{Nr}\epsilon$.
\end{lem}
%
%


\begin{thm}\label{thm:factorized_identifiability}
Consider the factorization-based LSR-DL problem \eqref{problem:factorized}. 
%
%
Suppose that Assumptions~\ref{assumptions:coefficient}--\ref{assumptions:parameter} are satisfied and $\frac{M_{\epsilon}}{M_{x}}<\frac{7}{2}(\bar\lambda C_{max}-\rho)$ with $\bar\lambda C_{\min}<\rho\leq \bar\lambda C_{\max}$. Then, the expected risk function $\E[\fyfac\big(\dknset\big)]$ has a local minimizer $([\dkn]^*)$ such that $\nor{\sum\bigotimes [\dkn]^*-\D^0}_F\leq \rho$.

Moreover, when the sample complexity requirements \eqref{sample_complexity_general} are satisfied for some positive constant $u$,
then with probability no less than $1-e^{-u}$ the empirical risk objective function $\fyfac\big(\dknset\big)$ has a local minimum achieved at $([\dkn]^*)$ such that $\nor{\sum\bigotimes[\dkn]^*-\D^0}_F\leq \rho$.
\end{thm}
\begin{proof}
Let us first consider the finite sample case. Theorem \ref{thm:finite_nonconvex} shows existence of a local minimizer $\D^*$ of Problem~\eqref{problem:sep_rank_NP} for constraint sets $\kn$, $\kr$, and $\cknr$, such that $\|\D^*-\D^0\|_F\leq \rho$ w.h.p. Here, we want to show that for such $\D^*$, there exists a $\{[\dkn]^*\}$ such that $\D^*=\sum\bigotimes[\dkn]^*$ and $\{[\dkn]^*\}$ is a local minimizer of Problem~\eqref{problem:factorized}.

First, let us consider Problem~\eqref{problem:sep_rank_NP} with $\cknr$. It is easy to show that any $\D\in \cknr$ can be written as $\sum_{k=1}^r\bigotimes\dkn$ such that for all $k \in [r]$ and $n \in [N]$ , without loss of generality,  $\dkn\in \alpha\ball_{m_n\times p_n}$ where $\alpha > \sqrt[N-1]{c}$.
Define
$
\mc C^{\mathrm{fac}}\triangleq \left\{\dkntup \big| \sum\bigotimes\dkn\in \cknr: \forall k,n, \dkn\in \alpha \mc \ball_{m\times p}  \right\}
$.
%
Since $\D^*\in \cknr$, there is a $([\dkn]^*)\in \mc C^{\mathrm{fac}}$ (with $\|[\dkn]^*\|_F= \sqrt[N-1]{c}$ for all $k \in [r]$ and $n \in [N]$) such that $\D^*=\sum\bigotimes[\dkn]^*$.
According to Lemma \ref{lem:neighborhoods_relation_D_to_Dkn}, 
for any $\dknset\in \mc C^{\mathrm{fac}}$ it follows from $\big\|\dkntup-([\dkn]^*)\big\|_F\leq \epsilon'$ that $\big\|\sum\bigotimes \dkn-\sum\bigotimes [\dkn]^*\big\|_F\leq \alpha^{N-1}\sqrt{Nr}\epsilon'$.  
Since $\D^*$ is a local minimizer of \eqref{problem:sep_rank_NP}, there exists a positive $\epsilon$ such that for all $\D\in \cknr$ satisfying $\nor{\D-\D^*}_F\leq \epsilon$, we have $\fy(\D^*)\leq \fy(\D)$. If we choose $\epsilon'$ small enough such that $\sqrt[N-1]{c}+\epsilon'\leq\alpha$ and  $\alpha^{N-1}\sqrt{Nr}\epsilon'\leq \epsilon$, then for any $\dkntup$ such that $\dkntup\in \mc C^{\mathrm{fac}}$ 
and $\big\|\dkntup-([\dkn]^*)\big\|_F\leq \epsilon'$, we have $\big\|\sum\bigotimes \dkn-\D^*\big\|_F\leq \epsilon$ and this means that
$\fyfac\big(\dknset\big)-\fyfac\big(\{[\dkn]^*\}\big)=\fy(\sum\bigotimes \dkn)-\fy(\D^*)\geq 0$. Therefore, $([\dkn]^*)$ is a local minimizer of Problem \eqref{problem:factorized}. This concludes our proof for the finite sample case with constraint set $\cknr$. 

Note that we can write $\kn={}^c{\mc K}^{N,1}_{\m m, \m p}$ and $\kr={}^c{\mc K}_{\m m, \m p}^{2,r}$ with $c\geq p$. Therefore, the above results also hold for $\kn$ and $\kr$ since they are special cases of $\cknr$.

It is easy to see similar relation exists between the local minima of $f_{\y}(\D)$ and $f_{\y}^{\mathrm{fac}}(\dknset)\triangleq\E[\fyfac(\dknset)]$, proving the asymptotic result in the statement of this theorem.
\end{proof}


\subsection{Discussion}\label{subsec:discussion}
In this section, we discuss the local identifiability of the true dictionary in the regularization-based formulation (Theorem \ref{thm:regularized_theorem}) and the factorization-based formulation (Theorem \ref{thm:factorized_identifiability}). For the regularization-based formulation, our results only hold for the case where the true dictionary is KS, i.e. $\D^0\in \kn$. We obtain sample complexity requirement of $\Omega(m p^3\rho^{-2})$ in this case, which matches the sample complexity requirement of the unstructured formulation~\cite{gribonval2014sparse}. This is due to the fact that the class of dictionaries in the regularization-based formulation is $\mc D_{m\times p}$, i.e., the LSR constraint is not explicitly imposed in this formulation. Thus, our covering number-based approach does not provide improved sample complexity results compared to the unstructured formulation. The experimental results of the regularization-based algorithm \stark~(see Section \ref{sec:experiments}) suggest there is room to improve our sample complexity result for the regularized formulation. We leave investigating such improvement as future work. Nonetheless, our results imply well-posedness of the regularized LSR-DL problem.

For the factorization-based formulation, we show that $\Omega(p^2\rho^{-2}r\sum_n m_np_n)$ samples are required for local identifiability of a dictionary of separation-rank $r$. This result matches that of our {rank-constrained} formulation stated in Theorem~\ref{thm:finite_nonconvex}.
Note that when the separation rank is $1$, this result gives a bound on the sample complexity of the KS-DL model as a special case. {To illustrate the implication of our bound, consider the case of $N=2, m_1 = m_2 = \sqrt{m}$ and $p_1 = p_2 = \sqrt{p}$. In this case, our bound scales as $p^2\sqrt{mp}$, which results in $\sqrt{mp}$ reduction in sample complexity scaling compared to the unstructured DL bound~\cite{gribonval2014sparse}.
In the case of $m_1=m,~m_2=1$ and $p_1= p,~p_2 = 1$ (unstructured DL), our bound scales as $mp^3$, which is consistent with the unstructured DL bound~\cite{gribonval2014sparse}.} Note that unlike the KS-DL analysis~\cite{shakeri2018achieve}, which shows a necessary sample complexity of $L= \max_{n\in\{1,\dots,N\}} \Omega( m_np_n^3\rho_n^{-2} )$, our analysis of the factorized model does not ensure identifiability of the true subdictionaries in the LSR-DL model. However, the results for KS-DL require the dictionary coefficient vectors to follow the separable sparsity model. In contrast, our result does not require any constraints on the sparsity pattern of coefficients.

\section{Computational Algorithms}\label{sec:algorithms}
In Section~\ref{sec:tractable_identif}, we showed that the tractable LSR-DL Problems~\eqref{problem:regularized} and~\eqref{problem:factorized} each have at least one local minimum close to the true dictionary. In this section we develop algorithms to find these local minima. Solving Problems \eqref{problem:regularized} and \eqref{problem:factorized} require minimization with respect to (w.r.t.) $\X\triangleq [\x_1^T,\cdots,\x_L^T]$. 
Therefore, similar to conventional DL algorithms, we introduce alternating minimization-type algorithms that at every iteration, first perform minimization of the objective function w.r.t. $\X$ (sparse coding stage) and then minimize the objective w.r.t. the dictionary (dictionary update stage).  
%
%
%
%

%
The sparse coding stage is a simple Lasso problem and remains the same in our algorithms. 
However, the algorithms differ in their dictionary update stages, which we discuss next.

\textit{Remark.} We leave the formal convergence results of our algorithms to future work. However, we provide a discussion on challenges and possible approaches to establish convergence of our algorithms in Appendix \ref{sec:convergence}.
\subsection{\stark: A Regularization-based LSR-DL Algorithm}\label{subsec:stark}

\begin{algorithm}[t]
\caption{Dictionary Update in \stark~for LSR-DL} \label{algo:stark}
\begin{algorithmic}[1]
\REQUIRE $\Y$, $\m \Pi$,~$\lambda_1>0$,~$\gamma>0$, $\X(t)$\footnotemark
%
%
\REPEAT
        \STATE Update $\Dp$ according to update rule \eqref{Dp_update}
	\FOR{$n \in [N]$}
        \STATE  Update ${\t W}_n$ according to \eqref{W_update}
    \ENDFOR
    \FOR{$n \in [N]$}
		\STATE  ${\t A}_n\leftarrow{\t A_n}- \gamma \pr{\mc \Dp-{\t W_n}}$
	\ENDFOR
%
%
%
\UNTIL{convergence}
\STATE Normalize columns of $\D$
\RETURN $\D(t+1)$
\end{algorithmic}
\end{algorithm}
\footnotetext{In the body of Algorithms \ref{algo:stark}--\ref{algo:osubdil} we drop the iteration index $t$ for simplicity.}

We first discuss an algorithm, which we term \textit{STructured dictionAry learning via Regularized low-ranK Tensor Recovery (\stark)}, that helps solve the regularized LSR-DL problem given in~\eqref{problem:regularized} and discussed in Section~\ref{sec:tractable_identif} using the Alternating Direction Method of Multipliers~(ADMM)~\cite{Boyd_ADMM_2011}.

The main novelty in solving~\eqref{problem:regularized} using $g_1(\Dp) = \strnorm{\Dp}$ is the dictionary update stage. This stage, which involves updating $\D$ for a fixed set of sparse codes $\X$, is particularly challenging for gradient-based methods because the dictionary update involves interdependent nuclear norms of different unfoldings of the rearranged tensor $\Dp$. Inspired by many works in the literature on low-rank tensor estimation~
\cite{Romera-Paredes_multilinear, Gandy_2011,Wimalawarne_2014}, we instead suggest the following reformulation of the dictionary update stage of~\eqref{problem:regularized}:
\begin{align}\label{dictionary_auxiliary_terms}
&\min_{\D\in \mc D, \t W_1, \cdots, \t W_N}~\frac{1}{2}\nor{\m Y-\m D\m X}_F^2+ \lambda_1\sum_{n=1}^N\trnorm{\m W_n^{(n)}}\nonumber\\
&\qquad\,\text{s.t.} \quad\quad \forall n\quad \t W_n=\t{D}^{\pi}.
\end{align}
In this formulation, although the nuclear norms depend on one another through the introduced constraint, we can decouple the minimization problem into separate subproblems. To solve this problem, we first find a solution to the problem without the constraint $\D\in \mc D$, then project the solution onto $\mc D$ by normalizing the columns of $\D$. We adopt this approximation to avoid the complexity of solving the problem with the constraint $\D\in \mc D$. {Such approach has been used in prior works; see, e.g.,~\cite{dantas2017learning,caiafa2013multidimensional}}. We can solve the objective function \eqref{dictionary_auxiliary_terms} (without $\D\in \mc D$ constraint) using ADMM, which involves decoupling the problem into independent subproblems by forming the following augmented Lagrangian:
\begin{align}\label{augmented_lagrange_function}
&\mc{L}_{\gamma}= \frac{1}{2}\nor{\m Y-\m D\m X}_F^2+ \sum_{n=1}^N \Big(\lambda_1\trnorm{\m W_n^{(n)}}\nonumber\\
& \qquad\qquad-\ip{\t A_n}{~\Dp-\t W_n}+\frac{\gamma}{2} \nor{~\Dp-\t W_n}_F^2 \Big),
\end{align}
where $\mc{L}_{\gamma}$ is shorthand for $\mc{L}_{\gamma}(\Dp, \{ \t W_n\}, \{ \t A_n\})$.
In order to find the gradient of \eqref{augmented_lagrange_function} with respect to $\Dp$, we rewrite the Lagrangian function in the following form:
\begin{align*}
&\mc{L}_{\gamma}= \frac{1}{2}\nor{\v y-\mc{T}(\t D^{\pi})}_2^2+ \sum_{n=1}^N  \Big( \lambda_1 \trnorm{\m W_n^{(n)}}\nonumber\\
&\qquad  \qquad- \ip{\t A_n}{~\Dp-\t W_n}+\frac{\gamma}{2} \nor{~\Dp-\t W_n}_F^2 \Big).
\end{align*}
Here, $\y\triangleq \vect(\Y)$ (not to be confused with our earlier use of $\y$ for $\vect(\t Y$)) and the linear operator $\mc T(\Dp)\triangleq \vect(\D\X)= \widetilde{\m X}^T \m \Pi^T \vect(\Dp)$, where $\widetilde{\m X}=\m X \otimes \m I_m $ and $\m \Pi$ is a permutation
matrix such that $\vect(\Dp)=\m \Pi \vect(\D)$. The procedure to find $\m\Pi$ is explained in Appendix~\ref{rearrangement_section}.

\textbf{ADMM Update Rules:}\label{ADMM_update}
Each iteration $\tau$ of ADMM consists of the following steps~\cite{Boyd_ADMM_2011}:
%
%
%
\begin{align}
&\t D^{\pi}(\tau)=\argmin\limits_{\t D^{\pi}} \mc{L}_{\gamma}  (\t D^{\pi}, \t W_n(\tau-1), \t A_n(\tau-1)),\label{D}\\
&\t W_n(\tau)=\argmin\limits_{\t W_n} \mc{L}_{\gamma}  (\t D^{\pi}(\tau), \t W_n, \t A_n(\tau-1)),\label{W}\\
&\t A_n(\tau)=\t A_n(\tau-1)- \gamma \pr{\Dp(\tau)-\t W_n(\tau)},\label{A}
\end{align}
for all $n \in [N]$.
%
%
The solution to \eqref{D} can be obtained by taking the gradient of $\mc{L}_{\gamma}(\cdot)$ w.r.t. $\t D^{\pi}$ and setting it to zero. Suppressing the iteration index $\tau$ for ease of notation, we have
\begin{align*}
\frac{\partial \mc{L_{\gamma}}}{\partial \Dp} = \mc{T}^*(\mc{T}(\Dp)-\v y)- \sum_{n=1}^N \t A_n +\sum_{n=1}^N \gamma \pr{\Dp-\t W_n},
\end{align*}
where $ \mc{T}^*(\v v)=\vect^{-1}\big(\m\Pi\widetilde{\X}\v v\big)$ is the \textit{adjoint} of the linear operator $\mc{T}$~\cite{Gandy_2011}. Setting the gradient to zero results in
\begin{align*}
&\mc{T}^*(\mc{T}(\Dp)) +\gamma N ~\Dp=\mc{T}^*(\v y) + \sum_{n=1}^N  \pr{\t A_n+\gamma \t W_n}.
\end{align*}
Equivalently, we have
\begin{align}
&\vect^{-1}\pr{\br{\m \Pi \widetilde{\X} \widetilde{\X}^T \m \Pi^T +\gamma N  \m I}\vect(\Dp)}\nonumber\\
&\qquad \qquad\quad =\vect^{-1}(\m \Pi \widetilde{\m X} \v y)+ \sum_{n=1}^N  \pr{\t A_n+\gamma \t W_n}.
\end{align}
Therefore, suppressing the index $\tau$, the update rule for $\Dp$ is
\begin{align}\label{Dp_update}
\Dp
=& \vect^{-1}\Big(\br{\m \Pi^{T} \widetilde{\m X} \widetilde{\m X}^T \m \Pi +\gamma N  \m I_{mp}}^{-1} \nonumber\\
&\cdot \Big[\m \Pi^{T} \widetilde{\m X} \v y+ \vect\Big(\sum_{n=1}^N  \pr{\t A_n
+\gamma\t W_n
}\Big)\Big]\Big).
\end{align}
To update $\{\t W_n\}$, we can further split~\eqref{W} into $N$ independent subproblems 
(suppressing the index $\tau$):
\begin{align}
\min_{\t W_n} ~\mc{L_W}=&\lambda_1 \trnorm{\m W_n^{(n)}} - \ip{\t A_n}{~\Dp-\t W_n}\nonumber\\
&+\frac{\gamma}{2} \nor{~\Dp-\t W_n}_F^2. \nonumber
\end{align}
%
We can reformulate $\mc{L_W}$ as
\begin{align*}
\mc{L_W}=
    &\lambda_1 \trnorm{\m W_n^{(n)}}
        +\frac{\gamma}{2} \Big\| \m W_n^{(n)}-\Big([\Dp]^{(n)}-\frac{\m A_n^{(n)}}{\gamma}\Big)\Big\|_F^2 \nonumber\\
&+\const
\end{align*}
The minimizer of $\mc{L_W}$ with respect to $\m W_n^{(n)}$ 
is
$
\shrink\pr{[\m D^{\pi}]^{(n)}-\frac{1}{\gamma} \m A_n^{(n)},~\frac{\lambda_1}{ \gamma}}
$
%
where 
$\shrink(\m A, z)$ applies soft thresholding at level $z$ on the singular values of matrix $\m A$~\cite{cai_2010_svt}. 
Therefore, suppressing the index $\tau$,
\begin{align}\label{W_update}
\t W_n
=&\mathrm{refold}\Big( \shrink\big([\m D^{\pi}]^{(n)}
%
-\frac{1}{\gamma}\m A^{(n)}_n
,~\frac{\lambda_1}{ \gamma}\big)\Big),
\end{align}
where $\refold(\cdot)$ is the inverse of the unfolding operator. Algorithm~\ref{algo:stark} summarizes this discussion and provides pseudocode for the dictionary update stage in \stark.

\subsection{\CPbased: A Factorization-based LSR-DL Algorithm} \label{subsec:tefdil}

While our experiments in Section~\ref{sec:experiments} validate good performance of \stark, the algorithm finds the dictionary $\D\in \mb R^{m \times p}$ and not the {subdictionaries} $\{\D_n\in \mb R^{m_n\times p_n}\}_{n=1}^N$. Moreover, \stark~only allows indirect control over the separation rank of the dictionary through the regularization parameter $\lambda_1$. This motivates developing a factorization-based LSR-DL algorithm that can find the subdictionaries and allows for direct tuning of the separation rank to control the number of parameters of the model.  
To this end, we propose a factorization-based LSR-DL algorithm termed \textit{Tensor Factorization-Based DL (\CPbased)} in this section for solving Problem \eqref{problem:factorized}.

We discussed earlier in Section \ref{subsec:stark} that the error term $\|\Y-\m D\X\|_F^2$ can be reformulated as $\|\y-\mc T(\Dp)\|^2$ where $\mc T(\Dp)=\widetilde{\X}^T \m \Pi^T\vect(\Dp)$. Thus, the dictionary update objective in~\eqref{problem:factorized} can be reformulated as $\|\y-\mc T(\sum_{k=1}^r \v d^k_N\circ\cdots\circ \v d^k_1)\|^2$ where $\v d^k_n\triangleq \vect(\dkn)$. To avoid the complexity of solving this problem, we resort to first obtaining an inexact solution by minimizing $\|\y-\mc T(\t A)\|^2$ over $\t A$ and then enforcing the low-rank structure by finding the rank-$r$ approximation of the minimizer of $\|\y-\mc T(\t A)\|^2$.  \CPbased~employs CP decomposition (CPD) to find this approximation and thus enforce LSR structure on the updated dictionary.


Assuming the matrix of sparse codes $\m X$ is full row-rank\footnote{In our experiments, we add $\delta\m I$ to $\X\X^T$ with a small $\delta>0$ at every iteration to ensure full-rankness.}, then $\widetilde{\m X}^T$ is full column-rank and $\t A=\mc T^+(\y)=\vect^{-1}\big(\m \Pi\big(\widetilde{\X}\widetilde{\X}^T\big)^{-1} \widetilde{\X} \y\big)$ minimizes $\|\y-\mc T(\t A)\|^2$. Now, it remains to solve the following problem to update $\{\v d^k_n\}$:
\begin{align*}
&\min_{\{\v d^k_n\}}~ \big\|\sum_{k=1}^r \v d^k_N\circ\cdots\circ \v d^k_1-\mc T^{+}(\v y)\big\|_F^2 .
\end{align*} 
%
Although finding the best rank-$r$ approximation ($r$-term CPD) of a tensor is ill-defined in general~\cite{DeSilva2008illposed}, various numerical algorithms exist in the tensor recovery literature to find a ``good'' rank-$r$ approximation of a tensor~\cite{kolda_tensor,DeSilva2008illposed}. \CPbased~can employ any of these algorithms to find the $r$-term CPD, denoted by $\mathrm{CPD}_r(\cdot)$, of $\mc T^{+}(\v y)$. At the end of each dictionary update stage, the columns of $\D=\sum\bigotimes \dkn$ are normalized.
Algorithm \ref{algo:tefdil} describes the dictionary update step of TeFDiL.

\begin{algorithm}[thb]
\caption{Dictionary Update in \CPbased~for LSR-DL} \label{algo:tefdil}
\begin{algorithmic}[1]
\REQUIRE $\Y$, $\X(t)$, $\m \Pi$, $r$
%
%
%
\STATE Construct $\mc T^+(\m y)=\vect^{-1}\big(\m \Pi\big(\widetilde{\X}\widetilde{\X}^T\big)^{-1} \widetilde{\X} \y\big)$
\STATE $\Dp\leftarrow \mathrm{CPD}_r(\mc T^{+}(\y))$
\STATE $\D\leftarrow\vect^{-1}\pr{\m \Pi^T\vect(\Dp)}$
\STATE Normalize columns of $\D$	
%

\RETURN $\D(t+1)$
\end{algorithmic}
\end{algorithm}



\vspace{-\baselineskip}
\subsection{O\subDL: An Online LSR-DL Algorithm}
Both \stark~and \CPbased~are batch methods in that they use the entire dataset for DL in every iteration. This makes them less scalable with the size of datasets due to high memory and per iteration computational cost and also makes them unsuitable for streaming data settings. 
To overcome these limitations, we now propose an online LSR-DL algorithm termed \textit{Online SubDictionary Learning for structured DL (O\subDL)} that uses only a single data sample (or a small mini-batch) in every iteration (see Algorithm~\ref{algo:osubdil}). 
This algorithm has better memory efficiency as it removes the need for storing all data points and has significantly lower per-iteration computational complexity. In \textit{O\subDL}, once a new sample $\t Y({t+1})$ arrives, its sparse representation $\t X({t+1})$ is found 
using the current dictionary estimate $\D(t)$ and then the dictionary is updated  using $\t Y({t+1})$ and $\t X({t+1})$. 
%
%
The dictionary update stage objective function after receiving the $T$-th sample is
%
\begin{align*}
J_T(\dknset)=\frac{1}{T} \sum\nolimits_{t=1}^T\|\y(t)-\big(\sum\nolimits_{k=1}^r\bigotimes\nolimits_{n=1}^N \D^k_n\big)\x(t)\|^2 . 
\end{align*}
%
We can rewrite this objective as 
\begin{align}
 J_T&=\sum\nolimits_{t=1}^T\|\Y^{(n)}(t)- \sum\nolimits_{k=1}^r \dkn \X^{(n)}(t) \m C_{n}^k(t)\|_F^2\nonumber\\
&=\sum\nolimits_{t=1}^T\|\widehat\Y^{(n)}(t)-\dkn \X^{(n)}(t) \m C_{n}^k(t)\|_F^2\nonumber\\
&= \tr\pr{[\dkn]^T\dkn \m A^k_{n}(t)}-2\tr\pr{[\dkn]^T \m B^k_{n}(t) }+\const,\nonumber
\end{align}
where, dropping the iteration index $t$, the matrix $\m C^k_{n}\triangleq \pr{\D^k_{N}\otimes \cdots \otimes\D^k_{n+1} \otimes\D^k_{n-1}\cdots\otimes\D^k_{1}}^T$ and the estimate $\widehat{\Y}^{(n)}\triangleq\Y^{(n)}- \sum_{\substack{i=1\\ i\neq k}}^r \D^i_{n} \X^{(n)} \m C_{n}^i$. We can further define the matrices 
$\m A^k_{n}(t)\triangleq \sum_{\tau=1}^t\X^{(n)}(t) \m C_{n}^k(\tau)[\m C_{n}^k(\tau)]^T [\X^{(n)}(\tau)]^T$ 
and
$\m B^k_{n}(t)\triangleq\sum_{\tau=1}^t\widehat\Y^{(n)}(\tau) [\m C_{n}^k(\tau)]^T [\X^{(n)}(\tau)]^T$. 
%
To minimize $J_T$ with respect to each $\dkn$, we take a similar approach as in Mairal et al.~\cite{mairal2010online} and use a (block) coordinate descent algorithm with warm start to update the columns of $\dkn$ in a cyclic manner. Algorithm \ref{algo:osubdil} describes the dictionary update step of O\subDL.
%
%

\begin{algorithm}[thb]
\caption{Dictionary Update in O\subDL~for LSR-DL} \label{algo:osubdil}
\begin{algorithmic}[1]
\REQUIRE $\t Y(t)$, $\{\m D^k_{n}(t)\}$, $\m A^k_{n}(t)$, $\m B^k_{n}(t)$, $\t X(t)$
%
%
%

\FORALL{$k \in [r]$}
	\FORALL{$n \in [N]$}
    	\STATE $\m C^k_{n}\leftarrow \pr{\D^k_{N}\otimes \cdots \otimes\D^k_{n+1} \otimes\D^k_{n-1}\cdots\otimes\D^k_{1}}^T$
        \STATE $\widehat\Y^{(n)}\leftarrow\Y^{(n)}- \sum_{\substack{i=1\\ i\neq k}}^r \D^i_{n} \X^{(n)} \m C_{n}^i$
    	\STATE $\m A^k_{n}\leftarrow\m A^k_{n}+ \X^{(n)} \m C_{n}^k[\m C_{n}^k]^T [\X^{(n)}]^T$
		\STATE $\m B^k_{n}\leftarrow\m B^k_{n}+ \widehat\Y^{(n)} [\m C_{n}^k]^T [\X^{(n)}]^T$
        %
 		
%
\FOR {$j=1,\cdots,p_n$}
	\STATE $[\dkn]_j\leftarrow\frac{1}{[\m A^k_n]_{jj}} ([\m B^k_n]_j-\dkn [\m A^k_n]_j)+[\dkn]_j$
	%
\ENDFOR
		

        %
    \ENDFOR		
\ENDFOR	
\STATE Normalize columns of $\D=\sum_{n=1}^r\bigotimes_{n=1}^N\dkn$
%

\RETURN $\{\dkn(t+1)\}$
\end{algorithmic}
\end{algorithm}
\vspace{-0.5\baselineskip}

%


\section{Numerical Experiments}\label{sec:experiments}

We evaluate our algorithms on synthetic and real-world datasets to understand the impact of training set size and noise level on the performance of LSR-DL. In particular, we want to understand the effect of exploiting additional structure in representation accuracy and denoising performance. We compare the performance of our proposed algorithms with existing DL algorithms in each scenario and show that in almost every case our proposed LSR-DL algorithms outperform $K$-SVD~\cite{aharon2006img}.  
Our results also offer insights into how the size and quality of training data can affect the choice of the proper DL model.
Specifically, our experiments on image denoising show that when the noise level in data is high, \CPbased~ performs best when the separation rank is $1$ but in low noise regimes, its performance improves as we increase the separation rank.  Furthermore, our synthetic experiments confirm that when the true underlying dictionary follows the KS (LSR) structure, our structured algorithms clearly outperform $K$-SVD, especially when the number of training samples is very small. This implies that our algorithms should perform well in applications where the true dictionary is close to being LSR-structured.

\textit{Remark.} In all our experiments, hyperparameters $\lambda_1$ and $\gamma$ (for \stark), $r$ (for \subDL), and regularization parameter for sparsity, $\lambda$, have been selected using cross-validation on each training dataset based on representation error. The only exception is for $r$ in cases where we specify its value in the KS-DL experiments ($r=1$) in Table~\ref{table:comparison} and in the \CPbased~experiments reported in Table~\ref{table:rank}.

\begin{table*}[htbp]
\caption{Performance of DL algorithms for image denoising in terms of PSNR}
\vspace{-0.4\baselineskip}
\label{table:comparison}
\centering
\begin{tabular}{|l|l|c||c|c|c||c|c|c|c|}
\cline{3-9}
\multicolumn{2}{c|}{} & Unstructured & \multicolumn{3}{|c||}{KS-DL ($r=1$)}&\multicolumn{3}{|c|}{LSR-DL ($r>1$)}\\
\hline
Image & Noise& $K$-SVD~\cite{aharon2006img} & SeDiL~\cite{hawe2013separable} & BCD~\cite{caiafa2013multidimensional} & \CPbased  & BCD & STARK  & \CPbased\\
\hline
\multirow{2}{*}{\texttt{House}}
& $\sigma = 10$ & 35.6697 & 23.1895 &   31.6089 & 36.2955 & 32.2952 & 33.4002  & 37.1275  \\
& $\sigma = 50$ & 25.4684 & 23.6916 &   24.8303 & 27.5412 & 21.6128  & 27.3945 & 26.5905 \\
\hline
\multirow{2}{*}{\texttt{Castle}}
& $\sigma = 10$ & 33.0910  & 23.6955   & 32.7592 & 34.5031 &  30.3561 &  37.0428 &  35.1000  \\
& $\sigma = 50$ & 22.4184 & 23.2658  & 22.3065 & 24.6670  &
20.4414 &  24.4965  & 23.3372  \\
\hline
\multirow{2}{*}{\texttt{Mushroom}}
& $\sigma = 10$ & 34.4957 & 25.8137& 33.2797 & 36.5382  &  32.2098 & 36.9443 &  37.7016  \\
& $\sigma = 50$ & 22.5495 & 22.9464 & 22.8554 & 22.9284 & 21.7792 & 25.1081 & 22.8374 \\
\hline
\multirow{2}{*}{\texttt{Lena}}
& $\sigma = 10$ & 33.2690  &  23.6605 &  30.9575 &  34.8854 & 31.1309 & 33.8813 & 35.3009\\
& $\sigma = 50$ & 22.5070 &  23.4207 &  21.6985 & 23.4988 & 19.5989 & 24.8211 & 23.1658 \\
\hline
\end{tabular}
\end{table*}
\begin{table*}[!ht]
\caption{Performance of TeFDiL with various ranks for image denoising in terms of PSNR}
\vspace{-0.4\baselineskip}
\label{table:rank}
\centering
\begin{tabular}{|l|l|c|c|c|c|c||c|c|}
\hline
Image & Noise& $r=1$ &  $r=4$ & $r=8$ & $r=16$ & $r=32$ & $K$-SVD\\
\hline
\multirow{2}{*}{\texttt{Mushroom}}
& $\sigma = 10$ & 36.5382  & 36.7538 &  37.4173  & 37.4906 & 37.7016  & 34.4957\\
& $\sigma = 50$ & 22.9284 &  22.8352 & 22.8384 & 22.8419 & 22.8374  & 22.5495\\
\hline
\multicolumn{2}{|l|}{Number of parameters} & 265  & 1060 & 2120 & 4240 & 8480 & 147456\\
\hline
\end{tabular}
\end{table*}

\textbf{Synthetic Experiments:} We compare  our algorithms to $K$-SVD (standard DL) as well as a simple block coordinate descent (BCD) algorithm that alternates between updating every subdictionary in problem \eqref{problem:factorized}. This BCD algorithm can be interpreted as an extension of the KS-DL algorithm~\cite{caiafa2013multidimensional} for the LSR model.
We show how structured DL algorithms outperform the unstructured algorithm $K$-SVD~\cite{aharon2006img} when the underlying dictionary is structured, especially when the training set is small.
We focus on 3rd-order tensor data and we randomly generate a KS dictionary $\D = \D_1 \otimes \m D_2 \otimes \m D_3$ with dimensions $\v m=[2,5,3]$ and  $\v p =[4,10,5]$. We select i.i.d samples from the standard Gaussian distribution, $\mc{N}(0,1)$, for the subdictionary elements, and then normalize the columns of the subdictionaries. To generate $\x$, we select the locations of $s=5$ nonzero elements uniformly at random. The values of those elements are sampled i.i.d.~from $\mc N(0,1)$. We assume observations are generated according to $\m y = \m D \m x$.
In the initialization stage of the algorithms, $\D$ is initialized using random columns of $\m Y$ for $K$-SVD and random columns of the unfoldings of $\m Y$ for the structured DL algorithms. Sparse coding is performed using OMP
\cite{Pati_93_omp}.
Due to the invariance of DL to column permutations in the dictionary, we choose reconstruction error as the performance criteria.
For $L=100$, $K$-SVD cannot be used since $p>L$.
Reconstruction errors are plotted in Figure~\ref{Synthetic_plot_3d}. It can be seen that for small number of samples, \CPbased~outperforms all three algorithms BCD, $K$-SVD, and \stark. As more samples become available, both \CPbased~and \stark~(the proposed algorithms) outperform BCD and $K$-SVD.

\begin{figure}[!ht]
\centering
 \begin{subfigure}{0.24\textwidth}
 \includegraphics[width=\textwidth]{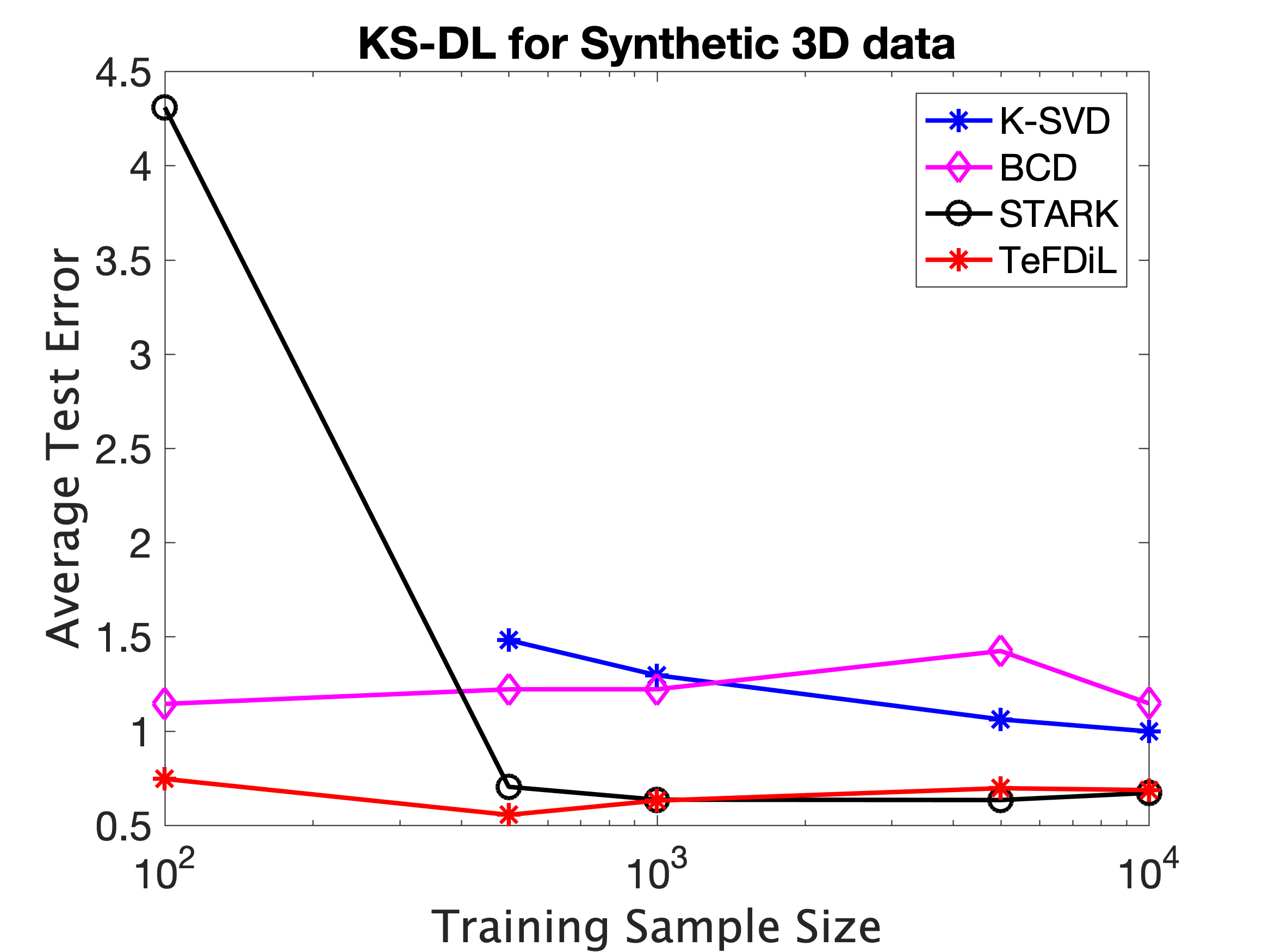}
 \caption{}
 \label{Synthetic_plot_3d}
 \end{subfigure}
 \begin{subfigure}{0.24\textwidth}
 \includegraphics[width=\textwidth]{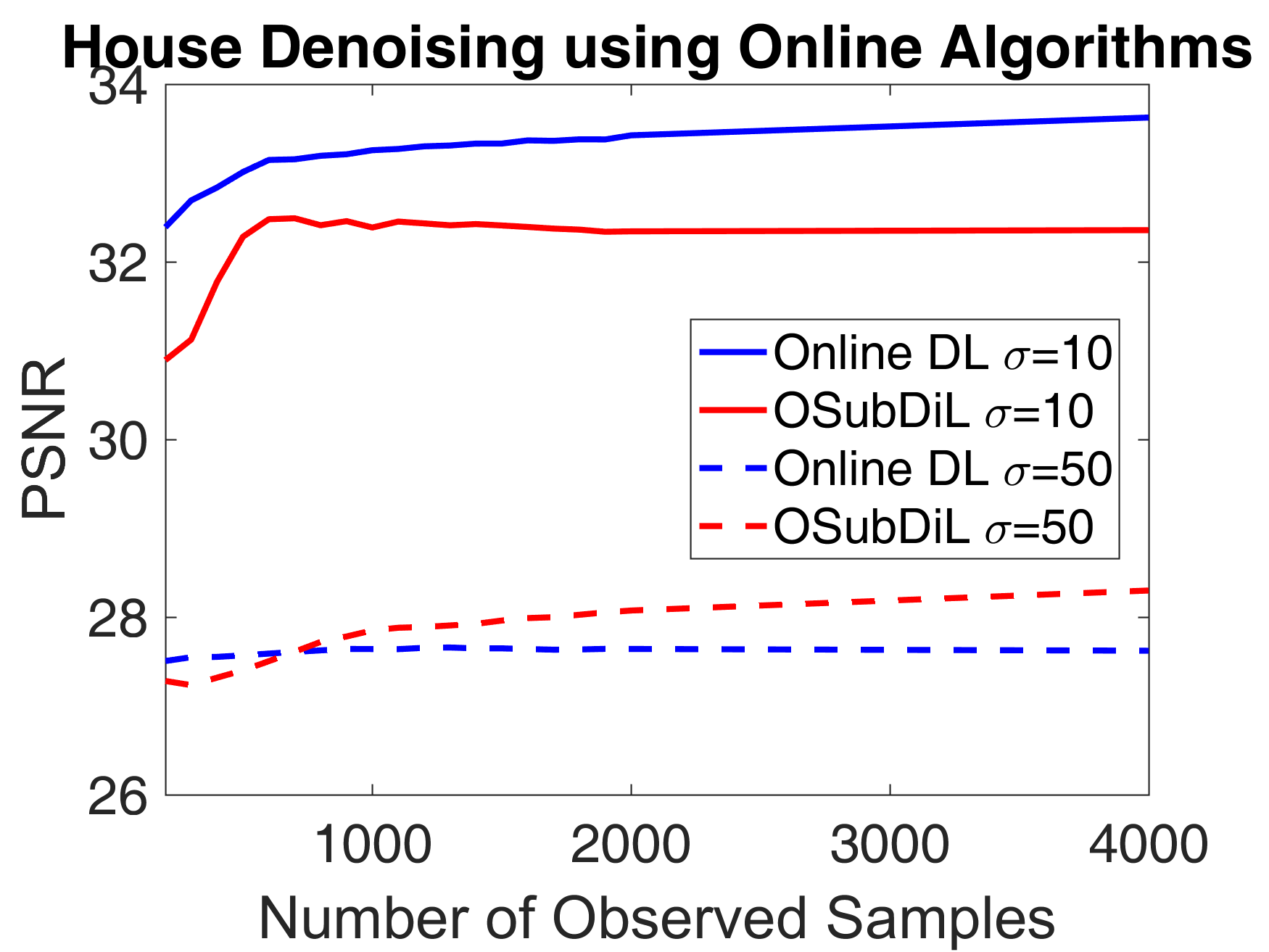}
 \caption{}
 \label{real_plot_peppers_online}
 \end{subfigure}
\caption{(a) Normalized representation error of various DL algorithms for 3rd-order synthetic tensor data. (b) Performance of online DL algorithms for \code{House}.} 
\label{Synthetic_plot}
\end{figure}
\textbf{Real-world Experiments:}
In this set of experiments, we evaluate the image denoising performance of different DL algorithms on four RGB images, \code{House}, \code{Castle}, \code{Mushroom}, and \code{Lena}, which have dimensions $256\times 256 \times 3$, $480\times 320\times 3$, $480\times 320\times 3$, and $512\times 512\times 3$, respectively. We corrupt the images using additive white Gaussian noise with standard deviations $\sigma=\{10, 50\}$.
To construct the training data set, we extract overlapping patches of size $8\times 8$  from each image and treat each patch as a 3-dimensional data sample.
We learn dictionaries with parameters $\v m =[3,8,8]$ and $\v p=[3,16,16]$. In the training stage, we perform sparse coding using FISTA~\cite{beck2009fista} (to reduce training time) with regularization parameter $\lambda = 0.1$ for all algorithms. To perform denoising, we use OMP with $s=\lceil p/20 \rceil$. 
To evaluate the denoising performances of the methods, we use the resulting peak signal to noise ratio (PSNR) of the reconstructed images~\cite{hore2010psnr}. Table~\ref{table:comparison} demonstrates the image denoising results.

\textbf{LSR-DL vs Unstructured DL:}
We observe that \stark~outperforms $K$-SVD in every case when the noise level is high and in most cases when the noise level is low. Moreover, \CPbased~outperforms $K$-SVD in both low-noise and high-noise regimes for all four images while having considerably fewer parameters (one to three orders of magnitude).\footnote{{While the improvements in image denoising reported in DL papers are sometimes below 0.5 dB~\cite{hawe2013separable,aharon2006img,zhang2016denoising,caiafa2013multidimensional}), we show that our algorithms provide 1--3 dB improvements over $K$-SVD in most scenarios.}}

\textbf{LSR-DL vs KS-DL:}
Our LSR-DL methods outperform SeDiL~\cite{hawe2013separable} and while BCD~\cite{caiafa2013multidimensional}~has a good performance for $\sigma=10$, its denoising performance suffers when noise level increases.\footnote{Note that SeDiL results may be improved by careful parameter tuning.}

Table~\ref{table:rank} demonstrates the image denoising performance of \CPbased~for \code{Mushroom} based on the separation rank of \CPbased. When the noise level is low, performance improves with increasing the separation rank. However, for higher noise level $\sigma=50$, increasing the number of parameters has an inverse effect on the generalization performance.

\textbf{Comparison of LSR-DL Algorithms:}
We compare LSR-DL algorithms BCD, STARK and \CPbased.
As for the merits of our LSR-DL algorithms over BCD, our experiments show that both \CPbased~and \stark~outperform BCD in both noise regimes.
In addition,
while \CPbased~ and \stark~can be easily and efficiently used for higher separation rank dictionaries, when the separation rank is higher, BCD~with higher rank does not perform well.
While \stark~has a better performance than \CPbased~ for some tasks, it has the disadvantage that it does not output the {subdictionaries} and does not allow for direct tuning of the separation rank. Ultimately, the choice between these two algorithms will be application dependent.
The flexibility in tuning the number of KS terms in the dictionary in \CPbased~(and indirectly in \stark, through parameter $\lambda_1$) allows selection of the number of parameters in accordance with the size and quality of the training data.
When the training set is small and noisy, smaller separation rank (perhaps $1$) results in a better performance. For training sets of larger size and better quality, increasing the separation rank allows for higher capacity to learn more complicated structures, resulting in a better performance.

\textbf{O\subDL~vs Online (Unstructured) DL:}
Figure~\ref{real_plot_peppers_online} shows the PSNR for reconstructing \code{House} using O\subDL~and Online DL in Mairal et al.~\cite{mairal2010online} based on the number of observed samples. 
We observe that in the presence of high level of noise, our structured algorithm is able to outperform its unstructured counterpart with considerably fewer parameters.





\section{Conclusion}\label{sec:conc}

We studied the low separation rank model (LSR-DL) to learn structured dictionaries for tensor data. This model bridges the gap between unstructured and separable dictionary learning (DL) models. For the intractable rank-constrained and the tractable factorization-based LSR-DL formulations, we show that given $\Omega\big(r(\sum_n m_n p_n) p^2 \rho^{-2}\big)$ data samples, the true dictionary can be locally recovered up to distance $\rho$. This is a reduction compared to the $\Omega(mp^3\rho^{-2})$ sample complexity of standard DL in Gribonval et al.~\cite{gribonval2014sparse}. However, a minimax lower bound scaling of $\Omega(p\sum_n m_np_n \rho^{-2})$ in Shakeri et al.~\cite{shakeri2016minimax} for KS-DL ($r=1$) has an $O(p)$ gap with our {sample complexity} upper bound. This gap suggests that the sample complexity bounds may be improved.
Possible future directions in this regard include finding minimax bounds for the LSR-DL model and tightening the gap between the sample complexity lower bound (minimax bound) and the upper bounds for this model.

We also show in the regularization-based formulation that $\Omega(mp^3\rho^{-2})$  samples are sufficient for local identifiability of the true Kronecker-structured (KS) dictionary up to distance $\rho$. Improving this result and providing sample complexity results for when the true dictionary is LSR (and not just KS) is also another interesting future work.

Finally, we presented two LSR-DL algorithms and showed that they have better generalization performance for image denoising in comparison to unstructured DL algorithm $K$-SVD \cite{aharon2006img} and existing KS-DL algorithms SeDiL~\cite{hawe2013separable} and BCD~\cite{caiafa2013multidimensional}. 
We also present O\subDL~that to the best our knowledge is the first online algorithm that results in LSR or KS dictionaries. We show that O\subDL~results in a faster reduction in the reconstruction error in terms of number of observed samples compared to the state-of-the-art  online DL algorithm~\cite{mairal2010online} when the noise level in data is high.



\begin{appendices}

\section{The Rearrangement Procedure}\label{rearrangement_section}
To illustrate the procedure that rearranges a KS matrix into a rank-$1$ tensor, let us first consider $\m A=\m A_1\otimes \m A_2$. The elements of $\m A$ can be rearranged to form $\m A^{\pi}=\v d_2 \circ \v d_1$, where $\v d_i=\vect(\m A_i)$ for $i=1,2$ \cite{van2000ubiquitous}. Figure \ref{figure_permutation} depicts this rearrangement for $\m A$.
Similarly, for $\m A=\m A_1\otimes \m A_2\otimes \m A_3$, we can write $\t{D}^{\pi}= \v d_3 \circ \v d_2\circ \v d_1$, where each frontal slice\footnote{A slice of a $3$-dimensional tensor is a $2$-dimensional section defined by fixing all but two of its indices. For example, a frontal slice is defined by fixing the third index.} of the tensor $\t{D}^{\pi}$ is a scaled copy of $\v d_3\circ \v d_2$.
The rearrangement of $\m A$ into $\t A^{\pi}$ is performed via a permutation matrix $\m\Pi$ such that $\vect(\t A^{\pi})=\m\Pi \vect(\m A)$.
Given index $l$ of $\vect(\m A)$ and the corresponding mapped index $l'$ of $\vect(\t A^{\pi})$, our strategy for finding the permutation matrix is to define $l'$ as a function of $l$. To this end, we first find the corresponding row and column indices $(i,j)$ of matrix $\m A$ from the $l$th element of $\vect(\m A)$. Then, we find the index of the element of interest on the $N$th order rearranged tensor $\t A^{\pi}$, and finally, we find its location $l'$ on $\vect(\t A^{\pi})$. Note that the permutation matrix needs to be computed only once in an offline manner, as it is only a function of the dimensions of the factor matrices and not the values of elements of $\m A$.

We now describe the rearrangement procedure in detail, starting with the more accessible case of KS matrices that are Kronecker product of $N=3$ factor matrices and then extending it to the general case. Throughout this section, we define an $n$-th order ``tile'' to be a scaled copy of $\m A_{N-n+1}\otimes \cdots \otimes \m A_N$ for $N>0$. A zeroth-order tile is just an element of a matrix. Moreover, we generalize the concept of slices of a $3$rd-order tensor to ``hyper-slices'': an $n$-th order hyper-slice is a scaled copy of $\v d_N \circ \v d_{N-1}\circ \cdots \circ \v d_{N-n+1}$.

\subsection{Kronecker Product of $3$ Matrices}
In the case of 3rd-order tensors, we take the following steps:
\begin{enumerate}[label=\roman*)]
\item Find index $(i,j)$ in $\m A$ that corresponds to the $l$-th element of $\vect(\m A)$.
\item Find the corresponding index $(r,c,s)$ on the third order tensor $\t A^{\pi}$.
\item Find the corresponding index $l'$ on $\vect(\t A^{\pi})$.
\item Set $\m \Pi(l',l)=1$.
\end{enumerate}

Let $\m A=\m A_1\otimes\m A_2 \otimes \m A_3$, with $\m A\in \mb R^{m\times p}$ and $\m A_i\in \mb R^{m_i\times p_i}$ for $i\in\{1,2,3\}$. For the first operation, we have
\begin{align}
(i,j)= \left(\ceil*{\frac{l}{m}},~ l-\floor*{\frac{l-1}{m}}m \right).
\end{align}

We can see from Figure \ref{figure_permutation_3} that the rearrangement procedure works in the following way. For each element indexed by $(i,j)$ on matrix $\m A$, find the 2nd-order tile to which it belongs. Let us index this 2nd-order tile by $T_2$. Then, find the $1$st-order tile (within the $2$nd-order tile indexed $T_2$) on which it lies and index this tile by $T_1$. Finally, index the location of the element (zeroth-order tile) within this first-order tile by $T_0$. After rearrangement, the location of this element on the rank-$1$ tensor is $(T_0,T_1,T_2)$.
\begin{figure}[t!]
\centering
 \includegraphics[width=\linewidth]{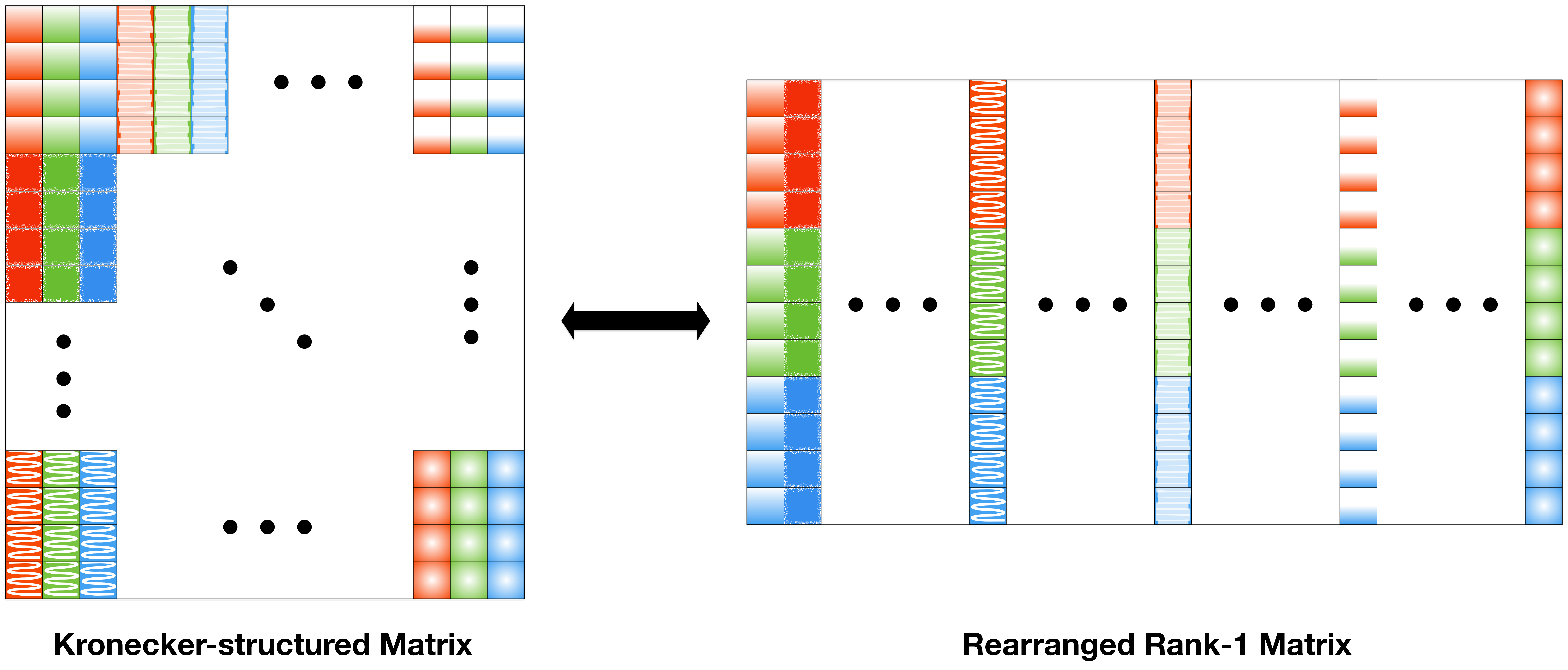}
\caption{Rearranging a Kronecker structured matrix ($N=2$) into a rank-1 matrix.}
\label{figure_permutation}
\end{figure}

In order to find $(T_0,T_1,T_2)$ that corresponds to $(i,j)$, we first find $T_2$, then $T_1$, and then $T_0$.  
To find $T_2$, we need to find the index of the $2$nd-order tile on which the element indexed by $(i,j)$ lies:
\begin{align}
T_2=\underbrace{\floor*{\frac{j-1}{p_2 p_3}}}_{S^2_j}m_1 + \underbrace{\floor*{\frac{i-1}{m_2 m_3}}}_{S^2_i} + 1,
\end{align}
where $S_j^2$ and $S_i^2$ are the number of the $2$nd-order tiles on the left and above the tile to which the element belongs, respectively. Now, we find the position of the element in this $2$nd-order tile:
\begin{align}
i_2=i-S^2_i m_2 m_3=i-\floor*{\frac{i-1}{m_2 m_3}}m_2 m_3,\nonumber\\
j_2=j- S^2_j p_2 p_3=j-\floor*{\frac{j-1}{p_2 p_3}} p_2 p_3.
\end{align}
For the column index, $T_1$, we have
\begin{align}
T_1=\underbrace{\floor*{\frac{j_2-1}{ p_3}}}_{S^1_j} m_2+\underbrace{\floor*{\frac{i_2-1}{ m_3}}}_{S^1_i} +1.
\end{align}
The location of the element on the $1$st-order tile is
\begin{align}
i_1=i_2-S^1_i m_3=i_2-\floor*{\frac{i_2-1}{ m_3}}m_3,\nonumber\\
j_1=j_2- S^1_j p_3=j_2-\floor*{\frac{j_2-1}{ p_3}} p_3 .
\end{align}
Therefore, $T_0$ can be expressed as
\begin{align}
T_0=\pr{j_1-1} m_3+i_1.
\end{align}

Finally, in the last step we find the corresponding index on $\vect(\t A^{\pi})$ using the following rule.
\begin{align}
l'=&(T_2-1) m_2 m_3 p_2 p_3 + (T_1-1) m_3 p_3 +T_0.
\end{align}

This process is illustrated in Figure \ref{figure_permutation_3}.

\subsection{The General Case}
We now extend our results to $N$-th order tensors. Vectorization and its adjoint operation are easy to compute for tensors of any order. We focus on rearranging elements of $\m A=\m A_1\otimes \m A_2 \otimes \cdots \otimes \m A_N$ to form the $N$-way rank-$1$ tensor $\t A^{\pi}$, where $\m A_n\in \mb R^{m_n\times p_n}$ for $n\in [N]$, $\m A\in \mb R^{m\times p}$, and $\t A^{\pi}\in \mb R^{m_N p_N\times m_{N-1} p_{N-1}\times \cdots \times m_1 p_1}$.

We first formally state the rearrangement and then we explain it. Similar to the case of $N=3$ explained earlier, for each element of the KS matrix $\m A$ indexed by $(i,j)$, we first find the $(N-1)$th-order tile to which it belongs, then  the $(N-2)$th-order tile, and so on. Let $T_{N-1},T_{N-2},\cdots,T_0$ denote the indices of these tiles, respectively. Then, after rearrangement, the element indexed $(i,j)$ on KS matrix $\m A$ becomes the element indexed $T_0,\cdots,T_{N-1}$ on the rearrangement tensor $\t A^{\pi}$.

Now, let us find the indices of the tiles of KS matrix $\m A$ to which the element $(i,j)$ belongs. In the following, we denote by $(i_n,j_n)$ the index of this element within its $n$th-order tile. Note that since $\m A$ is an $N$th-order tile itself, we can use $(i_N,j_N)$ instead of $(i,j)$ to refer to the index of the element on $\m A$ for consistency of notation. For the $(i_N,j_N)$-th element of $\m A$ we have
\begin{align*}
&T_{N-1}=\underbrace{\floor*{\frac{j_N-1}{\Pi_{t=2}^N ~p_t}}}_{S^N_j}m_1 + \underbrace{\floor*{\frac{i_N-1}{\Pi_{t=2}^N ~m_t}}}_{S^N_i} + 1,\nonumber\\
&i_{N-1}=i_N-S^N_i ~\Pi_{t=2}^N ~m_t,\nonumber\\
&j_{N-1}=j_N- S^N_j ~\Pi_{t=2}^N ~p_t,
\end{align*}
where $T_{N-1}$ is the index of the $(N-1)$-th order tile and $(i_{N-1},j_{N-1})$ is the location of the given element within this tile. Similarly, we have
 \begin{align*}
&T_{N-n}=\underbrace{\floor*{\frac{j_{N-n+1}-1}{\Pi_{t=n+1}^N ~p_t}}}_{S^{N-n+1}_j}m_n + \underbrace{\floor*{\frac{i_{N-n+1}-1}{\Pi_{t=n+1}^N ~m_t}}}_{S^{N-n+1}_i} + 1,\nonumber\\
&i_{N-n}=i_{N-n+1}-S^n_i ~\Pi_{t=n+1}^N ~m_t,\nonumber\\
&j_{N-n}=j_{N-n+1}- S^n_j ~\Pi_{t=n+1}^N ~p_t,
\end{align*}
for $N>n>1$. Finally, we have
\[
T_0=(j_1-1) m_N+i_1.
\]
%
It is now easy to see that the $(i_N,j_N)$-th element of $\m A$ is the $(T_0,~T_{1},~\cdots,~T_{N-1})$-th element of $\t A^{\pi}$.

%

Intuitively, notice that $N$-th order KS matrix $\m A$ is a tiling of $m_1\times p_1$ KS tiles of order $N-1$. In rearranging $\m A$ into $\t A^{\pi}$, the elements of each of these $(N-1)$-th order tiles construct a $(N-1)$-th order ``hyper-slice''. On matrix $\m A$, these tiles consist of $m_2\times p_2$ tiles, each of which is a $(N-2)$th-order KS matrix, whose elements are rearranged to a $(N-2)$-th hyper-slice of $\t A^{\pi}$, and so on. Hence, the idea is to use the correspondence between the $n$th-order tiles and $n$th-order hyper-slices: finding the index of the $n$-th order tile of $\m A$ on which $(i,j)$ lies is equivalent to finding the index of the $n$th-order hyper-slice of $\t A^{\pi}$ to which it is translated. Note that each entry of a tensor in indexed by an $N$-tuple and the index of an entry of a tensor on its $n$th hyper-slice is in fact its $n$th element in the index tuple of this entry. Therefore, we first find the $(N-1)$-th order KS tile of $\m A$  on which the $(i,j)$ element lies (equivalent to finding the $(N-1)$th-order hyper-slice to which $(i,j)$ is translated), and then find the location $(i_{N-1},j_{N-1})$ of this element on this tile. Next, the $(N-2)$-th order KS tile in which $(i_{N-1},j_{N-1})$ lies is found as well as the location $(i_{N-2},j_{N-2})$ of the element within this tile, and so on. 



\section{Proofs of Lemmas}\label{proofs_section}


\begin{proof}[Proof of Lemma \ref{lem:closedness_LSR}]
Proposition 4.1 in De Silva and Lim \cite{DeSilva2008illposed} shows that the space of tensors of order $N \geq 3$ and rank $r \geq  2$ is not closed. The fact that the rearrangement process preserves topological properties of sets means that the same result holds for the set $\knrall$ with $N \geq 3$ and rank $r \geq 2$. 

The proof for closedness of ${\mc L}_{\v m,\v p}^{N,1}$ and ${\mc L}_{\v m,\v p}^{2,r}$ follows from Propositions~4.2 and 4.3 in De Silva and Lim \cite{DeSilva2008illposed}, which can be adopted here due to the relation between the sets of low-rank tensors and LSR matrices.
\end{proof}

\begin{proof}[Proof of Lemma \ref{lem:unbounded_sequence}]
The rearrangement process allows us to borrow the results in Proposition 4.8 in De Silva and Lim \cite{DeSilva2008illposed} for tensors and apply them to LSR matrices.
\end{proof}

\begin{proof}[Proof of Lemma \ref{lem:covering_number_2nd}]
%
%
%
Define $\mc M_{m\times p}^r=\{ \D\in \ball| \rank(\D)\leq r\}$ and $\widehat{\mc L}^{2,r}_{\v m,\v p}={\mc L}_{\v m,\v p}^{2,r}\cap \ball$.
%
Since the rearrangement operator is an isometry w.r.t.~the Euclidean distance, the image of an $\epsilon$-net of $\widehat{\mc L}_{\v m,\v p}^{2,r}$ w.r.t. the Frobenius norm under this rearrangement operator is an $\epsilon$-net of $\mc M_{m'\times p'}^r$ ($m'=m_2p_2$ and $p'=m_1p_1$) w.r.t the Frobenius norm. Thus,
$
\N_F(\widehat{\mc L}^{2,r}_{\v m,\v p},\epsilon)=\N_F(\mc M_{m'\times p'}^r,\epsilon).
$
We also know that
$
\N_F(\mc M_{m'\times p'}^r,\epsilon)\leq (9/\epsilon)^{r(m'+p'+1)}
$~\cite{Candes2011tight}.
This means that 
\begin{align}\label{eq:cover_hat}
\N_F(\widehat{\mc L}^{2,r}_{\v m,\v p},\epsilon)\leq (9/\epsilon)^{r(m_1p_1+m_2p_2+1)}.
\end{align}
%
On the other hand, for the oblique manifold we have $\mc D_{m\times p}\subset p \ball$ and therefore, 
$\mc K^{2,r}_{\v m,\v p}\subset p\widehat{\mc L}^{2,r}_{\v m,\v p}$. Hence, 
$
\N_{\mxc}(\kr)\leq \N_{\mxc}(p\widehat{\mc L}^{2,r}_{\v m,\v p}, \epsilon).
$
Also, since $\nor{\m M}_{\mxc}\leq \nor{\m M}_F$ for any $\M$, it follows that an $\epsilon$-covering of any given set w.r.t. the Frobenius norm is also an $\epsilon$-covering of that set w.r.t. the max-column-norm. Thus 
$
\N_{\mxc}(\kr)\leq \N_{\mxc}(p\widehat{\mc L}^{2,r}_{\v m,\v p}, \epsilon)\leq \N_{F}(p\widehat{\mc L}^{2,r}_{\v m,\v p}, \epsilon).
$
%
Moreover, it follows from the fact $\N_F(p\widehat{\mc L}^{2,r}_{\v m,\v p},\epsilon)=\N_F(\widehat{\mc L}^{2,r}_{\v m,\v p},\epsilon/p)$ that
\begin{align}\label{eq:cover_epsilon_p}
\N_{\mxc}(\kr,\epsilon)\leq \mc N_{F}(\widehat{\mc L}^{2,r}_{\v m,\v p}, \epsilon/p).
\end{align}
Thus, from \eqref{eq:cover_hat} and \eqref{eq:cover_epsilon_p} we see that
$\N_{\mxc}(\kr,\epsilon) \leq (9p/\epsilon)^{r(m_1p_1+m_2p_2+1)}$,
%
which concludes the proof.
\end{proof}

\begin{proof}[Proof of Lemma \ref{lem:covering_number_general}]
Each element $\D\in \cknr$ can be written as a summation of at most $r$ KS matrices $\bigotimes \dkn$ such that $\nor{\bigotimes\dkn}_{F}\leq c$. This implies that $\cknr$ is a subset of the Minkowski sum (vector sum) of $r$ copies of ${}^{c}{\mc K}^{N,1}_{\v m, \v p}$, the set of KS matrices within the Euclidean ball of radius $c$. It is easy to show that the Minkowski sum of the $\epsilon$-coverings of $r$ sets is an $r\epsilon$-covering of the Minkowski sum of those sets in any norm. Therefore, we have 
\begin{align}\label{Minkowski}
\N_{\mxc}(\cknr, \epsilon)\leq \pr{\N_{\mxc}({}^{c}{\mc K}_{\v m,\v p}^{N,1}, \epsilon/{r})}^r.
\end{align}
Moreover, we have ${}^{c}{\mc K}_{\v m,\v p}^{N,1} \subset c \kn$. We also know from equation (16) 
that 
$
\N(\kn,\epsilon)\leq (3/\epsilon)^{\sum_{i=1}^N m_i p_i}
$.
Putting all these facts together, we get
\begin{align}
\N_{\mxc}(\cknr,\epsilon)&\leq \pr{\N_{\mxc}(c\kn, \epsilon/{r})}^r \nonumber\\
&\leq (3rc/\epsilon)^{r\sum_{i=1}^N m_i p_i}.
%
\end{align}
\end{proof}

\begin{proof}[Proof of Lemma \ref{lem:neighborhoods_relation_D_to_Dkn}] 
According to Lemma 2 in Shakeri et al.~\cite{shakeri2018achieve}, for any $\{\m A_n\}$ and $\{\m B_n\}$ we have
\begin{align}\label{kron_difference}
&\bigotimes\nolimits_{n=1}^N \m A_n - \bigotimes\nolimits_{n=1}^N \m B_n\nonumber\\
&\quad\quad= \sum\nolimits_{n=1}^N \m \Gamma_1\otimes\cdots\otimes(\m A_n-\m B_n)\otimes\cdots\otimes \m \Gamma_N,
\end{align}
where $\m \Gamma_n=\m A_n$ or $\m \Gamma_n=\m B_n$ depending on $n$. Let $\epsilon_n^k\triangleq\|\m A^k_n-\m B^k_n\|_F$. 
Using equality \eqref{kron_difference}, we have
\begin{align}
&\big\|\sum\nolimits_{k=1}^r\bigotimes \m A^k_n-\sum\nolimits_{k=1}^r\bigotimes \m B^k_n\big\|_F\nonumber\\
&=\big\|\sum_{k=1}^r \sum_{n=1}^N \m \Gamma^k_1\otimes\cdots\otimes(\m A^k_n-\m B^k_n)\otimes\cdots\otimes \m \Gamma^k_N \big\|_F\nonumber\\
&\leq \sum_{k=1}^r \sum_{n=1}^N  \big\|\m \Gamma^k_1\otimes\cdots\otimes(\m A^k_n-\m B^k_n)\otimes\cdots\otimes \m \Gamma^k_N \big\|_F\nonumber\\
&= \alpha^{N-1}\sum_{k=1}^r \sum_{n=1}^N  \epsilon^k_n\overset{(a)}{\leq} \alpha^{N-1}\sqrt{Nr} \epsilon,
\end{align}
where the inequality $(a)$ follows from $\|(\epsilon^k_n)\|_1\leq \sqrt{Nr}~\|(\epsilon^k_n)\|_2\leq\sqrt{Nr}\epsilon$.
\end{proof}

\section{Discussion on Convergence of the Algorithms}\label{sec:convergence}

%
%
%

The batch algorithms proposed in Section~\ref{sec:algorithms} are essentially variants of alternating minimization (AM). Establishing the convergence of AM-type algorithms in general is challenging and only known for limited cases. Here, we first present a well-known convergence result for AM-type algorithms in Lemma \ref{lem:BCD_convergencde} and discuss why our algorithms \stark~and \CPbased~do not satisfy the requirements of this lemma. Then, we show a possible approach for proving convergence of \stark. We do not discuss convergence analysis of O\subDL~here since it does not fall in the batch AM framework that we discuss here. We leave formal convergence results of our algorithms as open problems for future work.

First, let us state the following  standard convergence result for alternating minimization-type algorithms.
\begin{lem}[Proposition 2.7.1, \cite{bertsekas1999nonlinear}]\label{lem:BCD_convergencde}
Consider the problem 
\begin{align*}
\min_{\x = (\x_1,\dots,\x_M) \in \mc E=\mc E_1\times\mc E_2\times\cdots\times \mc E_M} f(\x),
\end{align*}
where $\mc E_i$ are closed convex subsets of the Euclidean space. Assume that $f(\cdot)$ is a continuous differentiable over the set $\mc E$. Suppose for each $i$ and all $\x\in \mc E$, the minimum
\begin{align*}
\min_{\xi\in \mc E_i} f(\x_1,\cdots,\x_{i-1},\xi,\x_{i+1},\cdots,\x_M)
\end{align*}
is uniquely attained. Then every limit point of the sequence $\{\x(t)\}$ generated by block coordinate descent method is a stationary point of $f(\cdot)$.
\end{lem}

The result of Lemma \ref{lem:BCD_convergencde} cannot be used for \CPbased~since its dictionary update stage does not have a unique minimizer (nonconvex minimization problem with multiple global minima)). Moreover, as discussed in Section \ref{subsec:tefdil}, \CPbased~only returns an inexact solution in the dictionary update stage. 

Similarly, this result cannot be used to show convergence of \stark~to a stationary point of Problem \eqref{problem:regularized} since, as discussed in Section \ref{subsec:stark}, \stark~returns an inexact solution in the dictionary update stage.  
However, we show next that dropping the unit column-norm constraint allows us to provide certain convergence guarantees. The unit column-norm constraint is essential in standard DL algorithms since in its absence, the $\ell_1$ norm regularization term encourages undesirable solutions where $\nor{\X}_F$ is very small while $\nor{\D}_F$ is very large. However, in the regularization-based LSR-DL problem, the additional regularization term $\strnorm{\Dp}$ ensures this does not happen. Therefore, dropping the unit column-norm constraint is sensible in this problem.

Let us discuss what guarantees we are able to obtain after relaxing the constraint set $\mc D_{m\times p}$. Consider the problem
\begin{align}\label{problem:regularized_DX}
\min_{\D\in \mb R^{m\times p},\X}
\nor{\m Y-\m D\m X}_F^2+ \lambda_1 \strnorm{\Dp}+\lambda \|\X\|_{\sumc}. 
\end{align}
We show in Proposition \ref{convergence_factorization_prop} that under the following assumptions, \stark~converges to a stationary point of Problem \eqref{problem:regularized_DX} (when the normalization step is not enforced). Then we discuss how this problem is related to Problem \eqref{problem:regularized}.

%

%
%


%
%
\begin{assum}\label{assupmtions:algorithm}
Consider the sequence $\big(\D(t),\X(t)\big)$ 
generated by \stark. 
We assume that for all $t\geq0$:
\begin{enumerate}[label=\Roman*)]
\item  Classical optimality conditions for the lasso problem (see Tibshirani~\cite{tibshirani2013lasso}) are satisfied.
\item  $\X(t)$ is full row-rank at all $t$.
\end{enumerate}
\end{assum}

%
\begin{prop}\label{convergence_factorization_prop}
Under Assumption~\ref{assupmtions:algorithm}, \stark~converges to a stationary point of problem \eqref{problem:regularized_DX}.
\end{prop}
\begin{proof} 
We invoke Lemma \ref{lem:BCD_convergencde} to show the convergence of \stark. To use this lemma, the minimization problem w.r.t.~each block needs to correspond to a closed convex constraint set and also needs to have a unique minimizer.

In the sparse coding stage, given Assumption \ref{assupmtions:algorithm}-I, the minimizer of the lasso problem is unique. 
%
%
In the dictionary update stage of \stark, the objective of problem \eqref{problem:regularized_DX} is strongly convex w.r.t. $\D$ under Assumption \ref{assupmtions:algorithm}-II and thus has a unique minimizer. Moreover, the constraint set $\mb R^{p\times L}$ is closed and convex. To utilize Lemma~\ref{lem:BCD_convergencde}, it remains to show that this minimum is actually attained by ADMM.
To this end, we restate Problem~\eqref{dictionary_auxiliary_terms} as
\begin{align}
&\min_{\Dp, \widetilde{\t W}} ~ f_1(\Dp)+f_2(\widetilde {\t W})
&\;\,\text{s.t.}\quad \widetilde{\t W}= \mc H \Dp,
\end{align}
where $f_1(\Dp)=\frac{1}{2}\nor{\m Y-\m D\m X}_F^2$ ($\m D\m X$ is a linear function of $\Dp$) and $f_2(\widetilde {\t W})=\lambda_1\sum_{n=1}^N\nor{(\t W_n)_{(n)}}_*$. It is clear that $\mc H\mc H^*$ is convertible. Therefore, according to~\cite{bertsekas1989parallel}[Chapter 3, Proposition 4.2], 
the ADMM algorithm converges to the unique minimizer of Problem \eqref{dictionary_auxiliary_terms}.

\end{proof}


So far we discussed convergence of \stark~to Problem \eqref{problem:regularized_DX} while our identifiability results are for problem \eqref{problem:regularized}. There is, however, a strong connection between minimization Problems \eqref{problem:regularized} and  \eqref{problem:regularized_DX}: for each local minimum $\widehat \D$ of problem \eqref{problem:regularized}, there exists an $\widehat \X$ such that $(\widehat \D,\widehat \X)$ is a local minimum of \eqref{problem:regularized_DX}. 
Define $\ellreg(\D,\X)=\frac{1}{L}\nor{\m Y-\m D\m X}_F^2+ \lambda_1 \strnorm{\Dp}+\frac{\lambda}{L} \|\X\|_{\sumc}$.
Consider any $\widehat \D$ that is a local minimum of \eqref{problem:regularized} and let $\widehat \X=\argmin_{\X\in \mb R^{p\times L}} \ellreg(\widehat \D,\X)$. We have $\ellreg(\widehat \D,\widehat \X)=\fyreg(\widehat \D)$. Since $\widehat \D$ is a local minimizer of $\fyreg(\D)$, $\fyreg(\widehat \D) \leq \fyreg(\D)$ for any $\D$ in the local neighborhood of $\widehat \D$. Also by definition, $\fyreg(\D) \leq \ellreg(\D,\X)$ for any $\X$. Thus, $\ellreg(\widehat \D,\widehat \X)\leq \ellreg(\D,\X)$ for any $(\D,\X)$ in the local neighborhood of $(\widehat \D,\widehat \X)$, meaning that $(\widehat \D,\widehat \X)$ is a local minimizer of \eqref{problem:regularized_DX}. Since we showed in Section~\ref{sec:tractable_identif} that a local minimum $\D^*$ of \eqref{problem:regularized} is close to the true dictionary $\D^0$, we can say there is a local minimum $(\D^*, \X^*)$ of \eqref{problem:regularized_DX} close to $\D^0$. So our recovery result for \eqref{problem:regularized} can apply to our proposed algorithm for solving \eqref{problem:regularized_DX} as well. 


\end{appendices}


\end{document}